%% file: ms.tex
\documentclass[letterpaper]{article} % DO NOT CHANGE THIS
\usepackage{aaai23}  % DO NOT CHANGE THIS
\usepackage{times}  % DO NOT CHANGE THIS
\usepackage{helvet}  % DO NOT CHANGE THIS
\usepackage{courier}  % DO NOT CHANGE THIS
\usepackage[hyphens]{url}  % DO NOT CHANGE THIS
\usepackage{graphicx} % DO NOT CHANGE THIS
\usepackage{natbib}  % DO NOT CHANGE THIS AND DO NOT ADD ANY OPTIONS TO IT
\usepackage{caption} % DO NOT CHANGE THIS AND DO NOT ADD ANY OPTIONS TO IT
\usepackage{algorithm}
\usepackage{algorithmic}

\usepackage[utf8]{inputenc} % allow utf-8 input
\usepackage[T1]{fontenc}    % use 8-bit T1 fonts
\usepackage{hyperref}       % hyperlinks
\usepackage{url}            % simple URL typesetting
\usepackage{booktabs}       % professional-quality tables
\usepackage{amsfonts}       % blackboard math symbols
\usepackage{nicefrac}       % compact symbols for 1/2, etc.
\usepackage{microtype}      % microtypography
\usepackage{xcolor}         % colors

\usepackage{graphicx}

\usepackage{amssymb}
\usepackage{bbding}
\usepackage{pifont}
\usepackage{amsthm}
\usepackage{mathtools}
\usepackage{bbm}
\usepackage{color}
\usepackage{macros}

\allowdisplaybreaks

\newtheorem{theorem}{Theorem}[section]
\newtheorem{lemma}[theorem]{Lemma}

\newcommand\numberthis{\addtocounter{equation}{1}\tag{\theequation}}

\usepackage{nomencl}
\makenomenclature

\renewcommand{\nompreamble}{The next list describes several symbols that will be later used within the body of the document}

\usepackage{newfloat}
\usepackage{listings}
\DeclareCaptionStyle{ruled}{labelfont=normalfont,labelsep=colon,strut=off} % DO NOT CHANGE THIS
\floatstyle{ruled}
\newfloat{listing}{tb}{lst}{}
\floatname{listing}{Listing}
\title{An Efficient Algorithm for Fair Multi-Agent Multi-Armed Bandit with Low Regret}
\author {
    Matthew Jones,\textsuperscript{\rm 1}
    Huy L\^{e} Nguy\~{\^{e}}n, \textsuperscript{\rm 1}
    Thy Nguyen \textsuperscript{\rm 1}
}
\affiliations {
    \textsuperscript{\rm 1} Khoury College, Northeastern University\\
    jones.m@northeastern.edu, hu.nguyen@northeastern.edu, nguyen.thy2@northeastern.edu
}

\usepackage{bibentry}
\begin{document}

\maketitle

\begin{abstract}
Recently a multi-agent variant of the classical multi-armed bandit was proposed to tackle fairness issues in online learning. Inspired by a long line of work in social choice and economics, the goal is to optimize the Nash social welfare instead of the total utility. Unfortunately previous algorithms either are not efficient or achieve sub-optimal regret in terms of the number of rounds $T$. We propose a new efficient algorithm with lower regret than even previous inefficient ones. For $N$ agents, $K$ arms, and $T$ rounds, our approach has a regret bound of $\tilde{O}(\sqrt{NKT} + NK)$. This is an improvement to the previous approach, which has regret bound of $\tilde{O}( \min(NK, \sqrt{N} K^{3/2})\sqrt{T})$. We also complement our efficient algorithm with an inefficient approach with  $\tilde{O}(\sqrt{\na\ts} + \np^2\na)$ regret. The experimental findings confirm the effectiveness of our efficient algorithm compared to the previous approaches.
\end{abstract}

\input{intro}

\input{relatedworks}
\input{preliminaries}

\input{ucb_AAAI2022}
%\input{AAAI2022/ucb2_AAAI2022}
\input{ucb2_postAAAI}
\input{experiments}

\bibliography{ms}
 \onecolumn
 \appendix
 \input{symbols}
 \input{appendix_post}
\end{document}

%% file: intro.tex
\section{Introduction}

The multi-armed bandit (MAB) problem poses an identical set of choices (or actions, \emph{arms}) at each time step to the decision-maker. When the decision maker pulls makes a choice, she receives a reward drawn from a distribution associated with that choice. The goal is to select arms to maximize the total reward in expectation, or to minimize their expected \emph{regret} with respect to the optimal strategy (or \emph{policy}). The MAB problem lends itself to many applications including allocation of resources in a financial portfolio \cite{hoffman2011portfolio, shen2015portfolio, huo2017risk}, selection of both dosages \cite{bastani2020online} and treatments \cite{durand2018contextual} in clinical healthcare,  and a broad range of recommender systems \cite{zhou2017large, bouneffouf2012contextual, bouneffouf2013contextual}. 
%designing customized instruction in an Intelligent Tutoring System \cite{clement2013multi},

In many scenarios, making a decision impacts not only one but multiple agents. For instance, consider a policy-maker making decisions that might influence various groups of their constituents \cite{zhu2018pension}, or a recommendation engine that picks hyperparameters for their system to serve many users. The multi-agent multi-armed bandit problem (MA-MAB), proposed in \citet{hossain2020fair},  is a variant of the MAB problem in which  there are $\np$ agents and $\na$ arms, and in every round $t$, the agents pull the same arm $a$ and receive rewards drawn from unknown distributions $\{D_{j,a}\}_{j=1}^\np$ with unknown means  $\{\pa^{\star}_{j,a}\}_{j=1}^\np$. The requirement for all agents to pull the same arm every round is representative of the aforementioned scenarios in which a larger entity makes decisions which affect many agents.

In this setting, one intuitive objective would be to optimize the expected rewards over all agents. This reduces the problem to the classic MAB problem and can be solved accordingly \cite{slivkins2019introduction}. In practice, this objective has a key weakness: individual rewards may be neglected in order to improve the reward of the collective. Consider an extremely divided case with two arms where just over half of agents receive reward 1 from the first arm and reward 0 from the second arm, and the other agents receive the opposite rewards. The optimal strategy in this case is to always choose the first arm: half of the agents receive reward 1, and almost half of the agents receive reward 0. While this formulation is well-studied by reduction, it unfortunately does not serve to enforce fairness among the agents. A fairer choice would be to pull each arm with roughly the same probability, which only slightly reduces the objective. In other words, the agents would receive similar rewards with each other if the algorithm learns a distribution $\po$ over the arms and converges to pulling each arm almost uniformly. Thus, the optimal strategy for a MA-MAB problem corresponds not to picking the optimal arm as in the classic MAB problem, but a probability distribution over the arms. 

In \cite{hossain2020fair}, the authors proposed an objective motivated by studies on fairness in social choice: the \textit{Nash social welfare} (NSW). This is a classical notion going back to \cite{Nash50} and \cite{KN79}.  For $n$ agents with expected utility $u_1,u_2,\ldots, u_n$, the Nash product $\prod_i u_i$ is sandwiched between $\left(\min_i u_i\right)^n \le \prod_i u_i \le \left(\frac{\sum_i u_i}{n}\right)^n$,  serving as a compromise between an egalitarian approach (optimizing the minimum reward across agents) and a utilitarian approach (optimizing the sum of rewards across agents). In the context of the MA-MAB problem, the optimal fair strategy corresponds to the  distribution that maximizes the cummulative Nash product of the agents' expected rewards. More formally,  $ \po^\star = \argmax_{\po \in \Delta^\na} \F(\po, \mu^\star) = \argmax_{\po \in \Delta^\na} \prod_{j \in [\np]} \left( \sum_{a \in [\na]} \po_a \mu^\star_{j,a} \right)$, where $\Delta^\na$ is the $\na$-simplex .

The prior work in \cite{hossain2020fair} proposed three algorithms for the MA-MAB problem (see table~\ref{tab:compare}). Their Explore-First and $\epsilon$-greedy algorithms have regret bounds that scale with $\Omega(\ts^{2/3})$. The algorithms involve computing the optimal policy given an estimated reward matrix, i.e., $\po = \argmax_{\po \in \Delta^\na} \F(\po, \hat{\mu})$.  Since the
objective is log-concave, the optimization step can be solved in polynomial time and there thus exist efficient implementations for the algorithms. Their UCB algorithm achieves the improved regret bound of $\widetilde{O}(\na\sqrt{\np\ts \cdot \min\{\np, \na\}})$. Note that the dependency on $\ts$ for the UCB algorithm is tight due to a  reduction of the MA-MAB problem with $\np = 1$ to the standard multi-armed bandits problem with a lower bound of $\Omega(\sqrt{\ts \na})$.  It is not clear to the authors of \cite{hossain2020fair} if the UCB algorithm “admits an efficient implementation.” This is due to the algorithm's step of computing $\argmax_{\po \in \Delta^\na} \F(\po, \hat{\mu}) + \sum_{a \in [\na]} \beta_a \po_a$, where $\beta_a$ is inversely proportional to the number of times the algorithm has picked arm $a$. With the added linear term, the optimization program is no longer log-concave. As such, previous algorithms for this problem in~\cite{hossain2020fair} either fail to achieve optimal dependency on $\ts$, or do not admit an efficient implementation.  This leads us to investigate the following question:\\

\textit{Is it possible to design an algorithm that admits an efficient implementation while achieving optimal dependency on $\ts$?}\\

% Most notably their UCB algorithm achieves a regret bound of $\widetilde{O}(\na\sqrt{\np\ts \cdot \min\{\np, \na\}})$ for $\np$ agents, $\na$ arms, and $\ts$ iterations. However, it is not clear to the authors of \cite{hossain2020fair} if it “admits an efficient implementation.” This is due to the algorithm's step of computing $\argmax_{\po \in \Delta^\na} \F(\po, \mu) + \sum_{j \in [\na]} \eta_j \po_j$, which is no longer a log-concave function because of the linear terms.  The other algorithms are efficient for all values of $N$ and $K$ but have regret bounds that scale with $\Omega(\ts^{2/3})$. The same work in \cite{hossain2020fair} also shows that the lower bound from the classic MAB problem holds and no algorithm in the MA-MAB problem can achieve an $o(\sqrt{KT})$ regret bound.

%rounds but unfortunately requires in each round the solution to an optimization problem which only has efficient approximation algorithms when $\np$ is fixed. 
\textbf{Our main contribution } is an efficient algorithm with  $\widetilde{O}(\sqrt{\np\na\ts} + \np\na)$ regret. We not only not give an affirmative answer to the question above, but also achieve an improved regret bound over the UCB algorithm of \cite{hossain2020fair} for most regimes of $\np, \na, \ts$. Our algorithm preserves the efficiency of the Explore-First and $\epsilon$-Greedy approaches in \cite{hossain2020fair}, while achieving an improved bound over the previous state-of-the-art.  We also complement our efficient algorithm with an inefficient approach with  $\widetilde{O}(\sqrt{\na\ts} + \np^2\na)$ regret.

%In the supplementary material, we also give an inefficient algorithm with  $\widetilde{O}(\sqrt{\na\ts} + \np^3\na)$ regret.

%The numerical experiments also confirm the effectiveness of our approach.

%On top of the improved regret bound, our algorithm can be implemented efficiently with standard convex optimization tools. 

% The Nash social welfare is also known to provide a number of preferable properties such as Pareto efficiency \cite{brandl2022funding} and optimality or approximate optimality in other contexts \cite{conitzer2017fair, caragiannis2019unreasonable, freeman2017fair}. We take direct advantage of the property that a direct scale on all mean rewards for a single agent will always directly scale the Nash social welfare by the same factor, so we can bound the domain of the individual rewards to the range [0,1] without affecting the optimal strategy. The Nash social welfare is also log-concave with respect to the selection strategy \cite{anari2016nash}, so for a fixed set of rewards we can find an optimal policy for Nash social welfare using classic convex optimization techniques.

\begin{table*}[ht]
\centering
\caption{Fair algorithms for the MA-MAB problem.}
\begin{tabular}{l l c l}
    \hline
    \textbf{Algorithm} & \textbf{Regret Bound} & \textbf{Efficient} & \textbf{Reference}\\ \hline \hline
    Explore-First & $\widetilde{O}\left(\sqrt[3]{NKT^2 \cdot \min\{N,K\}}\right)$ & \cmark  & \cite{hossain2020fair}, Theorem 1 \\ \hline
    $\epsilon$-Greedy & $\widetilde{O}\left(\sqrt[3]{NKT^2 \cdot \min\{N,K\}}\right)$ & \cmark  &  \cite{hossain2020fair}, Theorem 2 \\ \hline
    UCB & $\widetilde{O}\left(K\sqrt{NT \cdot \min\{N,K\}}\right)$ & \xmark & \cite{hossain2020fair}, Theorem 3 \\ \hline
    Algorithm \ref{alg:fair_UCB} & $\widetilde{O}\left(NK + \sqrt{NKT}\right)$ & \cmark & Theorem \ref{theorem:main_main}\\ \hline
        Algorithm \ref{alg:fair_UCB_inefficient} & $\widetilde{O}\left(N^2K + \sqrt{KT}\right)$ & \xmark & Theorem \ref{theorem:alg3_regret}\\ \hline\\
\end{tabular}
\label{tab:compare}
\end{table*}

%% file: relatedworks.tex
\section{Other related works}
Many variants of the multi-armed bandit problem have been proposed and studied such as adversarial bandit \cite{auer2002nonstochastic}, dueling bandit \cite{yue2009interactively,yue2012k}, Lipschitz bandit \cite{kleinberg2004nearly,flaxman2004online}, contextual bandit \cite{hazan2007online}, and sleeping bandit \cite{kleinberg2010regret}. Multi-agent variants of the problem have also been investigated in  \cite{landgren2016distributed, chakraborty2017coordinated, bargiacchi2018learning}. 

In the context of fair multi-armed bandit, \cite{joseph2016fair} proposes a framework where an arm with a higher expected reward is selected with a probability no lower than that of an arm with a lower expected reward. \cite{wang2020fairness} requires the fair policy to sample arms with probability proportional to the value of a merit function of its mean reward. \cite{liu2017calibrated} preserving fairness means  the probability of selecting each arm should be similar if the two arms have a similar quality distribution. \cite{gillen2018online} studies fairness in the linear contextual bandits setting where there are individual fairness constraints imposed by an unknown similarity metric. \cite{patil2020achieving} proposes a fair MAB variant that seeks to optimize the cumulative reward while also ensures that, at any round, each arm is pulled at least a specified fraction of times.

Our objective of finding a probability distribution over arms to optimize the Nash welfare objective can also be cast as the continuum-armed bandit problem where the Nash welfare function is the objective. \cite{kleinberg2019bandits} designs an algorithm with a regret bound of $\Tilde{O}\left(\ts^{\frac{\gamma+1}{\gamma+2}}\right)$, where $\gamma$, defined as the zooming dimension, would be $\Theta(\na)$ for the MA-MAB problem. The resulting bound would be no better than $O(\ts^{2/3})$ and approaches $O(\ts)$ as $K$ increases. It is also important to note that there is a long line of work on bandit convex optimization \cite{hazan2014bandit,bubeck2016multi,bubeck2017kernel, chen2019projection}. One can apply the approaches to optimize the $\log$ of the Nash welfare function. However, the regret bound of the new objective does not translate to that of the original objective.   

%% file: preliminaries.tex
\section{Preliminaries}
Define $[n] = \{1,\ldots,n\}$ for $n \in \mathbb{N}$. In the multi-agent multi-armed bandit problem, we have a set of agents  $[\np]$ and a set of arms $[\na]$ for $\np, \na \in \mathbb{N}$. For each agent $j \in [\np]$ and arm $a \in [\na]$, we have a reward distribution $D_{j,a}$ with mean $\pa^{\star}_{j,a}$ and support in $[0,1]$. Define $\tpa = (\tpa_{j,a})_{j \in [\np], a \in [\na]}$ as the mean reward matrix. At each round $t$, $a_t$ denotes the selected arm of round $t$ and  $r_{j,a_t,t} \sim D_{j,a_t}$ the realization of the reward distribution associated with arm $a_t$ of each agent $j$.  

%Define a policy as a probability distribution over $\na$ arms $\po \in \Delta^{\na},$ where $\Delta^{\na}$ is the $\na$ simplex.
For reward matrix $\mu = \mu_{j \in [\np], a \in [\na]} \in [0,1]^{\np \times \na}$ and policy $\po \in \Delta^{\na} $, we denote the Nash Social Welfare function by the product over the expected reward of each agent, $\F(\po, \mu) = \prod_{j \in [\np]} \left( \sum_{a \in [\na]} \po_a \mu_{j,a} \right)$. For a  time horizon $\ts$, our goal is to choose policies $\pi_t\in \Delta^{\na}$ at each round $t \in [\ts]$ to minimize the cumulative regret 
$ \TotalRegret_\ts = \sum_{t \in [\ts] } \F(\po^{\star}, \pa^{\star})  -  \sum_{t \in [\ts] }   \F(\po_t, \pa^{\star})$. We use $\pi_{a,t}$ to denote the probability of selecting arm $a$ under policy $\pi_t$. 
% Our algorithm maintains a confidence bound for the mean reward matrix $\tpa$ at every round. We have $\st_{j,a,t} = \sum_{\tau=1}^{t-1}\Indicator\{a_\tau=a\}$ and $\epa_{j,a,t} = \frac{1}{\st_{j,a,t}} \sum_{\tau=1}^{t-1} r_{j, a_\tau, \tau}  \Indicator\{a_\tau=a\} $ respectively denote the number of times arm $a$ has been  pulled by agent $j$ before round $t$ and the corresponding empirical mean. At round $t$, our algorithm samples an arm $a_t$ from $\po_t$ and pulls the same arm for every agent. We define $\st_{a,t}$ as the  number of times arm $a$ has been sample from the policy. We have $\st_{a,t} = \st_{j,a,t}$ for all $j \in [\np]$. At round $t$, the confidence bound for arm $a$ for agent $j$ is denoted by $w_{j,a,t}$, which will be specified later.  Define $\epa_t = (\epa_{j,a,t})_{j \in [\np], a \in [\na]}$ as the mean reward matrix and $w_t = (w_{j,a,t})_{j \in [\np], a \in [\na]}$ as the confidence bound matrix for round $t$.

At each round $t$, our algorithm samples an arm $a_t$ from $\po_t$ and pulls the same arm for every agent. Let $\st_{a,t} =  \sum_{\tau=1}^{t-1}\Indicator\{a_\tau=a\} $ denote the number of times arm $a$ has been sampled before round $t$, and $\epa_{j,a,t} = \frac{1}{\st_{a,t}} \sum_{\tau=1}^{t-1} r_{j, a_\tau, \tau}  \Indicator\{a_\tau=a\}$ the corresponding empirical mean of agent $j$'s reward on arm $a$.  Our main algorithm maintains a confidence bound for the mean reward matrix $\tpa$ at every round. Let $\epa_t = (\epa_{j,a,t})_{j \in [\np], a \in [\na]}$ denote the estimated mean reward matrix and $w_t = (w_{j,a,t})_{j \in [\np], a \in [\na]}$ the confidence bound matrix for round $t$.

%% file: ucb_AAAI2022.tex
\section{Algorithms}

\begin{algorithm}[t]
\begin{algorithmic}[1]
\STATE {\bf input: }{$\na,\np,\ts,\delta$}
\FOR{$t=1$ to $\ts$}
\IF{$t \leq \na$} 
\STATE $\po_t \leftarrow$ policy that puts probability 1 on arm $t$ 
\ELSE 
%\STATE $\UCB_t = \min(\epa_t + w_t,1)$
\STATE $\forall j,a, \epa_{j,a,t} = \frac{1}{\st_{a,t}} \sum_{\tau=1}^{t-1} r_{j, a_\tau, \tau}  \Indicator\{a_\tau=a\}$  %\quad \quad \# $\operatorname{Update} \epa_t$
%\STATE $\forall j,a,  w_{j,a,t} = \sqrt{\frac{12 (1 - \hat{\pa}_{j, a, t})  \logterm}{\st_{a,t}}}  +  \frac{12  \logterm}{\st_{a,t}}$  %\quad \quad \# $\operatorname{Update} w_t$
\STATE $\forall j,a, \UCB_{j,a,t} = \min(\epa_{j,a,t} + w_{j,a,t},1)$ %\quad \quad \# $\operatorname{Update} \UCB_t$
%\STATE $\UCB_t = \min(\epa_t + w_t, 1) = $ 
\STATE $\po_t \leftarrow \operatorname{argmax}_{\po \in \Delta^\na} \F(\po, \UCB_t)  $ %= \operatorname{argmax}_{\po \in \Delta^\na} \log \F(\po, \UCB_t)$
\ENDIF
\STATE Sample $a_t$ from $\po_t$
\STATE Observe rewards $\{r_{j,a_t,t}\}_{j \in \np}$
\STATE $\st_{a_t, t+1} \leftarrow \st_{a_t, t} + 1$
%\STATE Update $\epa_{t+1}$ and $w_{t+1}$
\ENDFOR
\end{algorithmic}
\caption{Fair multi-agent UCB algorithm}
\label{alg:fair_UCB}
\end{algorithm}
Our UCB algorithm, algorithm~\ref{alg:fair_UCB}, first selects each arm once in order to obtain an initial estimate $\hat{\mu}$ for $\mu$. Then, it computes the upper confidence bound estimate $U_t = \{\UCB_{j,a,t}\}_{j \in [\np], a \in [\na]}$ of the true mean for each arms $a$ of $\np$ agents and finds a policy $\po_t$ that optimizes the Nash social welfare function given $U_t$.  Due to the log-concavity of Nash social welfare, we use standard convex optimization tools to optimize $\log(\F(\po, \UCB_t))$ and simultaneously optimize $\F(\po, \UCB_t)$.

This approach differs from the UCB algorithm in \citet{hossain2020fair} in two aspects.  First,  the optimization step in their UCB algorithm uses an additive regularization term in the objective rather than on the estimate $\hat{\mu}$, and is therefore not log-concave.  Second, our confidence interval is defined  in terms of the empirical mean and as a result our confidence interval is about a factor of  $\sqrt{1 - \epa_{j,a,t}}$ tighter than that of \citet{hossain2020fair}. As our proof involves bounding the regret by the Lipschitz continuity of the Nash social welfare, our confidence interval allows for a careful analysis of the algorithm.

Note that the UCB algorithm in \citet{hossain2020fair} is horizon-independent. Although the algorithm~\ref{alg:fair_UCB} requires the value of the time horizon $\ts$ as input, it can be easily modified to be horizon-independent. One approach is to modify the confidence interval $w_{j,a,t}$ so it becomes a function of the current time step, $t$ , rather than the time horizon, $\ts$. Specifically, both $\ln(4NKT/ \delta)$ terms of $w_{j,a,t}$ in algorithm~\ref{alg:fair_UCB} would become $\ln(8NKt^2/ \delta)$. Lemma~\ref{lemma:LPfair_UCB_param_in_CR_main} and theorem~\ref{theorem:main_main} can be easily adapted to the new confidence interval. The horizon-independent variant of our algorithm would have the same regret bound of $\widetilde{O}\left(NK + \sqrt{NKT}\right)$.

\textbf{Overall approach:} Our objective is to bound the regret $ \sum_{t\in [\ts]} \F(\po^{\star}, \tpa) - \F(\po_t, \tpa)$. Observe that $\F(\cdot,\cdot)$ is monotone in the second argument so $\F(\po^{\star}, \UCB_t)\ge \F(\po^{\star}, \tpa)$. By the optimality of $\po_t$, we also have $\F(\po_t, \UCB_t)\ge \F(\po^{\star}, \UCB_t)$. Thus, we can reduce the problem to bounding $\sum_{t\in [\ts]} \F(\po_t, \UCB_t) - \F(\po_t, \tpa)$.

The key idea to bound the regret in a single round $t$ is to look at the expected reward of each agent $j$ if the mean reward was $\UCB_t$. Formally, let $g_{j,t} = \sum_{a\in [\na]} \po_{a, t} (1-\UCB_{j,a,t})=1-\sum_{a\in [\na]} \po_{a, t} \UCB_{j,a,t}$. If there are a lot of agents with large $g_{j,t}$, then the Nash product $\F(\po_t, \UCB_t)=\prod_j \left(1-g_{j,t}\right)$ is small and the regret is therefore small. More precisely, we consider two cases depending on whether there exists a $p \geq 0$ such that the set of agents $\{j\in [\np] : g_{j,t} \geq  2^{-p}\}$ is of size at least $3 \cdot 2^p\ln (\ts) $. If such a $p$ exists, then 
\[
\F(\po_t, \UCB_t)  \leq (1 - 2^{-p})^{3\cdot 2^p \ln \ts} \leq \frac{1}{\ts^3}
\]
As a result, the regret of round $t$ is negligible. 

We now need to bound the regret of rounds where no such $p$ exists. The key idea is to show that when no $p$ exists, the upper confidence bounds are on average very close to the true means. For intuition, suppose the following similar statement holds. Let $g'_{j,t}=\sum_{a\in [\na]} \po_{a, t} (1- \hat{\pa}_{j,a,t})$ and for all $p\ge 0$, the set of agents $\{j\in [\np] : g'_{j,t} \geq  2^{-p}\}$ is of size at most $3 \cdot 2^p\ln (\ts)$. Given this condition we can bound
\begin{align*}
    \sum_{j \in [\np]} g'_{j,t} &\leq \int_0^1 [\text{number of agents $j$ s.t. } g'_{j,t} \geq  x]dx \\ & \le 1 + 6\ln \ts \log \np 
\end{align*}

Notice that our estimation error $w_{j,a,t}$ is a function of $1-\hat{\pa}_{j,a,t}$ and we showed that $1-\hat{\pa}_{j,a,t}$ is small on average. Thus, by making a careful averaging argument, we can show that the upper bound $\UCB_t$ is close to $\tpa$ on average. The regret bound then follows from the smoothness of the function $\F(\pi_t,\cdot)$. The actual proof has to overcome additional technical challenges due to the difference between the desired $g'_{j,t}$ and the actual $g_{,t}$.

\subsection{Analysis}

We seek to bound $\sum_{t\in [\ts]} \F(\po_t, \UCB_t) - \F(\po_t, \tpa)$ using
 the smoothness of the $\F$ objective. We include the missing proof of the  Lipschitz-continuity property and  the other lemmas in the appendix. 
 \begin{lemma}
\label{lemma:lipschitz_main} (Lemma 3, \cite{hossain2020fair})
Given a policy $\po \in \Delta^k$ and reward matrices $\mu^1, \mu^2 \in [0,1]^{\np \times \na}$, we have  
\[
\left\lvert\F(\po,\mu^1) - \F(\po,\mu^2)\right\rvert \leq \sum_{j \in [\np]} \sum_{a \in [\na]} \po_{a}  \left\lvert \mu^1_{j,a} - \mu^2_{j,a} \right\rvert
\]
\end{lemma}

Lemma~\ref{lemma:lipschitz_main} implies that  a Lipschitz-continuity analysis for bounding $\sum_{t\in [\ts]} \F(\po_t, \UCB_t) - \F(\po_t, \tpa)$ would benefit from a tight confidence bound on the means of the rewards. The following lemma proves a confidence interval that is about factor of  $\sqrt{1 - \epa_{j,a,t}}$ tighter than that of \citet{hossain2020fair}.
\begin{lemma}
\label{lemma:LPfair_UCB_param_in_CR_main}

For any $\delta\in(0,1)$, with probability at least $1-\delta/2$, $\forall$ $t>\na,a\in[\na]$, $j\in[\np]$, $\left\lvert\pa^{\star}_{j, a} - \hat{\pa}_{j, a, t} \right\rvert \leq  \sqrt{\frac{12 (1 - \hat{\pa}_{j, a, t}) \logterm}{\st_{a,t}}}  +  \frac{12  \logterm}{\st_{a,t}} = w_{j,a,t}$. 
\end{lemma}

Our confidence bound $w_{j,a,t}$ in lemma~\ref{lemma:LPfair_UCB_param_in_CR_main} has a $ \widetilde{O }\left(\sqrt{\frac{1 - \hat{\pa}_{j, a, t}}{\st_{a,t}}}\right)$  term and a  $\widetilde{O} \left( \frac{1}{\st_{a,t}} \right)$ term. Using Young's inequality, we can bound both the first term by $\widetilde{O }\left(\left(1 - \hat{\pa}_{j, a, t}\right) + \frac{1}{\st_{a,t}}  \right)$, and thus we have

\begin{align}
\label{youngs}
    w_{j,a,t} \in \widetilde{O }\left(\left(1 - \hat{\pa}_{j, a, t}\right) + \frac{1}{\st_{a,t}}  \right). 
\end{align}

 As mentioned above, if there are a lot of agents with large $g_{j,t}$ at round $t$, then the regret will be negligible. Thus, our main goal is to analyze the regret for rounds $t$'s when there are not enough such agents.  The following lemma formalizes the intuition and bounds the sum of expected empirical reward over all agents by the confidence intervals $w_t$. In other words, it bounds the $\left(1 - \hat{\pa}_{j, a, t}\right)$ term of equation~\ref{youngs} over all actions $a$ and agents $j$.
\begin{lemma}
\label{lemma:highreward_main}
Define  $g_{j,t} = \sum_{a\in [\na]} \po_{a, t} \left(1-\UCB_{j,a,t}\right)$, and $S(t,p) = \left\{j \text{ for } j \in [\np] :  g_{j,t} \geq  2^{-p} \right\}.$  If  $\left\lvert S(t,p) \right\rvert < 2^p \cdot 3\ln \ts $ for all $p \geq 0 $, then 
 %\[\sum_{j \in [\np]} \sum_{a \in [\na]}  \po_{a, t} (1-\hat{\pa}_{j,a,t}) \leq  3\ln\ts(1+2\log n) + \sum_{j \in [\np]} \sum_{a \in [\na]}   \po_{a,t} w_{j,a,t}   \]
 \begin{align*}
     &\sum_{j \in [\np]} \sum_{a \in [\na]}  \po_{a, t} (1-\hat{\pa}_{j,a,t}) \\  &\leq 1 + 6\ln \ts \log \np + \sum_{j \in [\np]} \sum_{a \in [\na]}   \po_{a,t} w_{j,a,t}.
 \end{align*}
\end{lemma}

 Lemma~\ref{lemma:LPfair_UCB_concentration_width_nonsqr_main} bounds the error incurred by the $\frac{1}{\st_{a,t}} $ term of equation~\ref{youngs}.  We carefully analyze the martingale sequence to obtain a tighter bound than that of black-box approaches, i.e. in lemma~\ref{lemma:LPfair_UCB_concentration_width_nonsqr_main} we bound the martingale sequence to be of  $O(\na \ln \ts / \na )$ while Azuma-Hoeffding inequality would give us $O(\sqrt{\ts\na})$.

\begin{lemma}
\label{lemma:LPfair_UCB_concentration_width_nonsqr_main}
With probability $1-\delta/2$,\[ 
\sum_{t \in [\ts]} \sum_{a \in [\na]} \po_{a,t} / \st_{a,t} \leq  2 \na\left(\ln\frac{\ts}{\na} + 1  \right)+\ln (2/\delta).\]
\end{lemma} 

With lemma~\ref{lemma:highreward_main} and lemma~\ref{lemma:LPfair_UCB_concentration_width_nonsqr_main}, we are ready to bound $\sum_{t\in [\ts]} \F(\po_t, \UCB_t) - \F(\po_t, \tpa)$ by  a Lipschitz-continuity analysis
of the Nash social welfare function. % We sketch the proof here and include the full version in the appendix. 
\begin{theorem}
 Suppose $\forall j, a, t,$  $\Reward_{j, t,a}\in[0,1]$ and $w_{j,a,t} =  \sqrt{\frac{12 (1 - \hat{\pa}_{j, a, t}) \logterm}{\st_{a,t}}}  +  \frac{12  \logterm}{\st_{a,t}} $, for any $\delta\in(0,1)$, the  regret of the Fair multi-agent UCB algorithm (Algorithm \ref{alg:fair_UCB}) is $\TotalRegret_\ts =\widetilde{O}\left( \sqrt{\np\na\ts} + \np\na\right)$ with probability at least $1-\delta$. 
 \label{theorem:main_main} 
\end{theorem}
\begin{proof}
Let $ I' \subseteq [\ts]$ denote the set of all rounds $t$ where there exists $p \geq 0$ such that $\left\lvert S(t,p) \right\rvert \geq 2^{p} 3\ln\ts $. Let  $I = [\na] \cup I'$. We have
\begin{align*}
  &  \sum_{t \in I} \F(\po_t, \UCB_t) - \F(\po_t, \pa^{\star}) \\ & \leq \sum_{t \in I'} \F(\po_t, \UCB_t)  + \na \\ 
    & \le \sum_{t \in I'} \prod_{j \in S(t,p) } (1 - g_{j,t}) + \na \\
    & \leq \ts \left(1-2^{-p}\right)^{2^{p}\cdot3\ln T} + \na \\ 
    &\le\frac{1}{T^{2}} + \na. 
    \numberthis \label{ez_main}
\end{align*}

The last inequality is due to the fact that $(1 - \frac{1}{x})^x \leq \frac{1}{e}~\forall x\ge 1$. We bound the regret of rounds not in $I$. For any $\delta \in (0,1)$, the events in lemma~\ref{lemma:LPfair_UCB_param_in_CR_main} and lemma~\ref{lemma:LPfair_UCB_concentration_width_nonsqr_main} hold with probability at least $1-\delta,$
\begin{align*}
    &\sum_{t \notin I} \F(\po_t, \UCB_t) - \F(\po_t, \pa^{\star}) \\ 
    &\le\sum_{t \notin I} \sum_{j \in [\np]}  \sum_{a \in [\na]} \po_{a,t} \left\lvert \UCB_{j,a,t} - \pa^{\star}_{j,a,t} \right\rvert \\ 
    &=\sum_{t \notin I} \sum_{j \in [\np]}  \sum_{a \in [\na]} \po_{a,t} \left\lvert \UCB_{j,a,t} - \hat{\pa}_{j,a,t} + \hat{\pa}_{j,a,t} - \pa^{\star}_{j,a,t} \right\rvert \\ 
    &\leq 2 \sum_{t \notin I} \sum_{j \in [\np]}  \sum_{a \in [\na]} \po_{a,t}  \sqrt{\frac{12 (1 - \hat{\pa}_{j, a, t}) \logterm}{\st_{a,t}}} \\ & \quad +  \sum_{t \notin I} \sum_{j \in [\np]}  \sum_{a \in [\na]} \po_{a,t}  \frac{24  \logterm}{\st_{a,t}}.\numberthis \label{hard_main} 
\end{align*}
The first inequality follows from  lemma~\ref{lemma:lipschitz_main}. The last inequality follows from lemma~\ref{lemma:LPfair_UCB_param_in_CR_main} and the definition of $\UCB_{j,a,t}$.

By lemma~\ref{lemma:LPfair_UCB_concentration_width_nonsqr_main}, we can bound the error incurred by the linear term of the confidence interval in equation~\ref{hard_main},
\begin{align*}
 \sum_{t \notin I} & \sum_{j \in [\np]}  \sum_{a \in [\na]} \po_{a,t}  \frac{24 \logterm}{\st_{a,t}} \\ & \leq  \sum_{t \in [\ts]} \sum_{j \in [\np]}  \sum_{a \in [\na]}  \po_{a,t}   \frac{24  \logterm}{\st_{a,t}} \\ 
  &  =  \sum_{j \in [\np]}   \sum_{t \in [\ts]}  \sum_{a \in [\na]}   \po_{a,t}  \frac{24  \logterm}{\st_{a,t}} \\ &\leq 24  \logterm \np \left(2 \na\left(\ln\frac{\ts}{\na} + 1  \right)+\ln (2/\delta) \right).
  \numberthis \label{hard1_main}
\end{align*}
We are done after bounding the remaining term of equation~\ref{hard_main}.  For brevity, we bound it without the log term,  
\begin{align*}
    & \sum_{t \notin I} \sum_{j \in [\np]}  \sum_{a \in [\na]}\po_{a,t} \sqrt{\frac{ 1 - \hat{\pa}_{j, a, t} }{\st_{a,t}}} \\ & \quad \leq \sum_{t \notin I} \sum_{j \in [\np]}  \sum_{a \in [\na]} \po_{a,t} \left(\frac{ q(1 - \hat{\pa}_{j, a, t})}{2} + \frac{1}{2q\cdot \st_{a,t}}\right). \numberthis \label{exposition1} 
    % \\ 
    % & \leq  \sum_{t \notin I} \frac{\sqrt{\na\np}}{(2\sqrt{\na\np}+2\sqrt{\ts})  \sqrt{12 \logterm} } \left(1 + 6\ln \ts + \sum_{j \in [\np]} \sum_{a \in [\na]}   \po_{a,t} w_{j,a,t} \right) \\ &\quad+ \frac{\sqrt{12 \logterm}}{2} \left(\frac{\np\sqrt{T}}{\sqrt{\np\na}}+\np\right)  \left(2 \na\left(\ln\frac{\ts}{\na} + 1  \right)+\ln (2/\delta)\right) 
    \end{align*}
The inequality follows from Young's inequality for $q \geq 0$. Applying lemma~\ref{lemma:highreward_main} to the first term and lemma~\ref{lemma:LPfair_UCB_concentration_width_nonsqr_main} to the second term of equation \ref{exposition1}, we have:
\begin{align*}
  &  \sum_{t \notin I} \sum_{j \in [\np]}  \sum_{a \in [\na]}\po_{a,t} \sqrt{\frac{ 1 - \hat{\pa}_{j, a, t} }{\st_{a,t}}} \\ &  
  \quad \leq  \frac q2 \sum_{t \notin I} \left( 1 + 6\ln \ts + \sum_{j \in [\np]} \sum_{a \in [\na]}   \po_{a,t} w_{j,a,t} \right) \\ & \quad + \frac{\np}{2q}  \left(2 \na\left(\ln\frac{\ts}{\na} + 1  \right)+\ln (2/\delta)\right) \\ 
   & \quad \leq  \frac q2 \sum_{t \notin I}  \sum_{j \in [\np]} \sum_{a \in [\na]}   \sqrt{\frac{12 (1 - \hat{\pa}_{j, a, t}) \logterm}{\st_{a,t}}} \\ & \quad  + \frac q2 \sum_{t \notin I}  \sum_{j \in [\np]} \sum_{a \in [\na]}    \frac{12  \logterm}{\st_{a,t}} + 6  q \cdot \ts \cdot \ln \ts \\ & \quad + \frac{\np}{2q}  \left(2 \na\left(\ln\frac{\ts}{\na} + 1  \right)+\ln (2/\delta)\right),
\end{align*} where we use the fact that $6\ln \ts \geq 6 \ln 2 \geq 1$. Suppose $q \in (0, 1]$, applying lemma~\ref{lemma:LPfair_UCB_concentration_width_nonsqr_main} to the linear term of the confidence interval, we have 
\begin{align*}
      & \sum_{t \notin I} \sum_{j \in [\np]}  \sum_{a \in [\na]}\po_{a,t} \sqrt{\frac{ 1 - \hat{\pa}_{j, a, t} }{\st_{a,t}}} \\   
        & \quad  \leq   \frac q2 \sum_{t \notin I}  \sum_{j \in [\np]} \sum_{a \in [\na]}   \sqrt{\frac{12 (1 - \hat{\pa}_{j, a, t}) \logterm}{\st_{a,t}}}  \\ & \quad + \frac{\np \logterm}{q}  \left(2 \na\left(\ln\frac{\ts}{\na} + 1  \right)+\ln (2/\delta)  \right)
        \\ & \quad +  6  q \cdot \ts \cdot \ln \ts.
\end{align*}
Setting $q = \frac{\sqrt{\na\np}}{(\sqrt{\na\np}+\sqrt{\ts})  \sqrt{12 \logterm} } \leq 1$ and re-arranging the terms, we have 
\begin{align*}
    & \frac 12 \sum_{t \notin I} \sum_{j \in [\np]}  \sum_{a \in [\na]}\po_{a,t} \sqrt{\frac{ 1 - \hat{\pa}_{j, a, t} }{\st_{a,t}}} 
    \\& \quad \leq \left(\logterm \right)^{3/2}\left(1+ \frac{\sqrt{\ts}}{\sqrt{\na \np}} \right) 2 \np \na\left(\ln\frac{\ts}{\na} + 1  \right) 
    \\& \quad +  \left(\logterm \right)^{3/2}\left(1+ \frac{\sqrt{\ts}}{\sqrt{\na \np}} \right) \np  \ln (2/\delta) 
     \\ & \quad +  \frac{6\ts \sqrt{\na\np}}{(\sqrt{\na\np}+\sqrt{\ts})}. \numberthis \label{hard2_main} 
\end{align*}
From equations~\ref{ez_main},~\ref{hard_main},~\ref{hard1_main},~\ref{hard2_main}, we have
\begin{align*}
  &   \sum_{t \in [\ts]} \F(\po_t, \UCB_t) - \F(\po_t, \pa^{\star}) \\ &\quad = O\left((\sqrt{\np\na \ts} + 
    \np\na) \cdot \operatorname{polylog}(\np \na\ts / \delta) \right) .
\end{align*}

By monotonicity of the Nash-social welfare function and the optimization step in the algorithm, we have 
$
    \F(\po_t, \UCB_t) \geq \F(\po^{\star}, \UCB_t) \geq \F(\po^{\star}, \pa^{\star})  
$
. Thus,
\begin{align*}
     & \sum_{t \in [\ts] } \F(\po^{\star}, \pa^{\star})  -  \sum_{t \in [\ts] }   \F(\po_t, \pa^{\star}) \\ & \quad = O\left((\sqrt{\np\na \ts} + 
    \np\na) \cdot \operatorname{polylog}(\np \na\ts / \delta) \right) 
\end{align*}
\end{proof}

 %\left(2 \na\left(\ln\frac{\ts}{\na} + 1  \right)+\ln (2/\delta)  \right)

%% file: ucb2_postAAAI.tex
\subsection{An Inefficient Algorithm with Improved Regret}

\begin{algorithm}[ht]
\begin{algorithmic}[1]
\STATE {\bf input: }{$\na,\np,\ts,\delta$}
\FOR{$t=1$ to $\ts$}
\IF{$t \leq 180N^2K \ln(6NTK/\delta)\ln T$} 
\STATE $\po_t \leftarrow$ policy that puts probability 1 on arm $\lceil t/(180N^2 \ln(6NTK/\delta))\ln T\rceil$ 
\ELSE 
%\STATE $\UCB_t = \min(\epa_t + w_t,1)$
\STATE $\forall j,a, \epa_{j,a,t} = \frac{1}{\st_{a,t}} \sum_{\tau=1}^{t-1} r_{j, a_\tau, \tau}  \Indicator\{a_\tau=a\}$  %\quad \quad \# $\operatorname{Update} \epa_t$
%\STATE $\forall j,a,  w_{j,a,t} = \sqrt{\frac{12 (1 - \hat{\pa}_{j, a, t})  \logterm}{\st_{a,t}}}  +  \frac{12  \logterm}{\st_{a,t}}$  %\quad \quad \# $\operatorname{Update} w_t$
\STATE \(
    S_\pi =\{\pi \in \Delta^K : \sum_{a\in[K], j\in[N]}\pi_a(1-\hat{\mu}_{j,a,t}) \le 1 + 2\ln T\}
\)
\IF{$S_\pi \ne \emptyset$}
\STATE $\po_t \leftarrow \operatorname{argmax}_{\po \in S_\pi} \left(\F(\po, \epa_t) + \po\cdot \eta_{t}\right)  $ %= \operatorname{argmax}_{\po \in \Delta^\na} \log \F(\po, \UCB_t)$
\ELSE
\STATE $\pi_t \leftarrow$ policy that puts probability 1 on a random arm $a \in [K]$
\ENDIF
\ENDIF
\STATE Sample $a_t$ from $\po_t$
\STATE Observe rewards $\{r_{j,a_t,t}\}_{j \in \np}$
\STATE $\st_{a_t, t+1} \leftarrow \st_{a_t, t} + 1$
%\STATE Update $\epa_{t+1}$ and $w_{t+1}$
\ENDFOR
\end{algorithmic}
\caption{Fair multi-agent UCB algorithm with high start-up cost}
\label{alg:fair_UCB_inefficient}
\end{algorithm}

Algorithm \ref{alg:fair_UCB_inefficient} is able to obtain a tighter regret bound in terms of $\ts$ by first pulling each arm $\np^2$ times, then selecting $\po_t$ to optimize the Nash Social Welfare on $\epa_t$ plus an additive term $\po\cdot \eta_{t}$, where $\eta_{t} = \{\eta_{a,t}\}_{a \in [\na]}$ is a vector. The additive term gives an  upper bound on $|\F(\po,\mu^*)-\F(\po,\epa_t)|$, and by the optimization step we obtain that for all $t>\tilde{O}(\np^2\na)$,
\begin{align*}
    \F(\po^*,\mu^*) &\le \F(\po^*,\epa_t)+\sum_{a\in[\na]}\po^*_a \cdot \eta_{a,t}\\
    &\le \F(\po_t,\epa_t)+\sum_{a\in[\na]}\po_{a,t} \cdot \eta_{a,t}\\
    &\le \F(\po_t,\pa^*)+2\sum_{a\in[\na]}\po_{a,t} \cdot \eta_{a,t}.
\end{align*}
Thus, by bounding $\sum_{t\in[\tilde{O}(\np^2\na),T]}\sum_{a\in[\na]}\po_{a,t}\cdot \eta_{a,t}$, we obtain a bound on the total regret:
 
\begin{theorem}
 Suppose $\forall j, a, t,$  $\Reward_{j, t,a}\in[0,1]$, $w_{j,a,t} =  \sqrt{\frac{12 (1 - \hat{\pa}_{j, a, t}) \ln(6NKT/\delta)}{\st_{a,t}}}  +  \frac{12  \ln(6NKT/\delta)}{\st_{a,t}} $, and \begin{align*}
  \eta_{a,t} =& \tilde{O}(\sqrt{K/T})\sum_{j\in[N]}(1-\hat{\mu}_{j,a,t} \\
  &+\tilde{O}(\sqrt{T/K} + \sqrt{N})\frac{1}{N_{a,t}}\\
  &+ O(1/\sqrt{N})\sum_{j\in[N]}w_{j,a,t}
  \end{align*} for any $\delta\in(0,1)$, the  regret of Algorithm \ref{alg:fair_UCB_inefficient} is $\TotalRegret_\ts =\widetilde{O}\left( \sqrt{\na\ts} + \np^2\na\right)$ with probability at least $1-\delta$. 
 \label{theorem:alg3_regret} 
\end{theorem}
We defer the proof of this theorem and the details of the constants and log terms in the $O$ and $\tilde{O}$ of $\eta$ to the appendix. At a high level, this bound comes from using a tighter bound in place of lemma \ref{lemma:LPfair_UCB_param_in_CR_main}, where we bound $\left|\sum_{j\in[N]}\pa^*_{j,a}-\hat{\pa}_{j,a,t}\right|$ instead of $\left|\pa^*_{j,a}-\hat{\pa}_{j,a,t}\right|$. We also analyze the regret at each time step $t$ using \[\F(\pi,\hat{\pa}_t) = \prod_{j\in [\np]}\left(\mathbb{E}_{a\sim\pi}\mu^*_{j,a} + \mathbb{E}_{a\sim\pi}(\hat{\mu}_{j,a,t}-\mu^*_{j,a})\right).\] Since $\F(\pi,\pa^*) = \prod_{j\in [\np]}\mathbb{E}_{a\sim\pi}\mu^*_{j,a}$, we can bound $\F(\pi,\pa^*) - \F(\pi,\hat{\pa}_t)$ by \[
     \sum_{m=1}^N\sum_{\{B\subseteq[N]:|B|=m\}}\prod_{j\in B}\frac{\mathbb{E}_{a\sim\pi}(\hat{\mu}_{j,a,t}-\mu^*_{j,a})}{\mathbb{E}_{a\sim\pi}\mu^*_{j,a}},
\] 
which drops the leading factor $\prod_j\mathbb{E}_{a\sim\pi}\mu^*_{j,a}\le1$ from the bound. Note that here the regret is bounded by $\mathbb{E}_{a\sim\pi}(\hat{\mu}_{j,a,t}-\mu^*_{j,a})$ as opposed to $\mathbb{E}_{a\sim\pi}|\hat{\mu}_{j,a,t}-\mu^*_{j,a}|$ in lemma~\ref{lemma:lipschitz_main}. Analyzing the terms with $m=1$ and $m\ge2$ separately allows us to derive the bound $\eta_{a,t}$ in our algorithm. 

%Our regret analysis uses a variant of lemma~\ref{lemma:highreward_main} modified to use $\hat{\mu_t}$ instead of $\UCB_t$, lemma \ref{lemma:LPfair_UCB_concentration_width_nonsqr_main}, and the bound on $\sum_{t,j\in[\np],a\in[\na]} \po_{a,t}w_{j,a,t}$ from theorem \ref{theorem:main_main} to achieve the regret bound.

The algorithm obtains a bound in $\ts$ which matches the known lower bound $O(\sqrt{KT})$ up to logarithmic terms, at the tradeoff of a high initial cost for pulling each arm $\tilde{O}(\np^2)$ times. Additionally, similar to the UCB algorithm of \cite{hossain2020fair} the algorithm uses an additive regularization term in its optimization step and therefore does not have a known efficient implementation. 

%% file: experiments.tex
\section{Experiments}

\begin{figure*}[ht]
    \centering
    \begin{tabular}{c c c}
        \includegraphics[width=4.3cm]{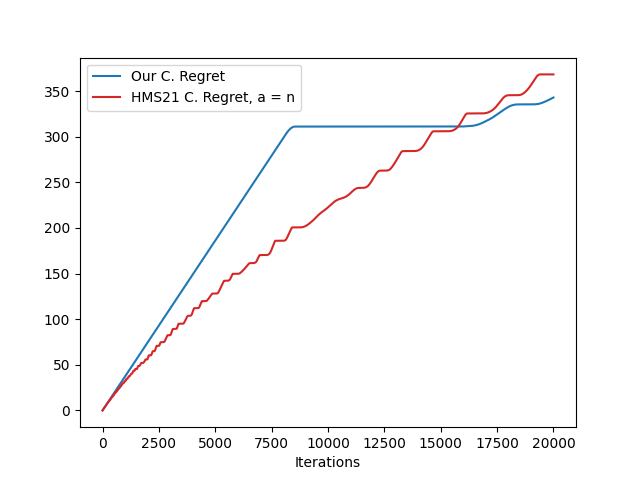} &
        \includegraphics[width=4.3cm]{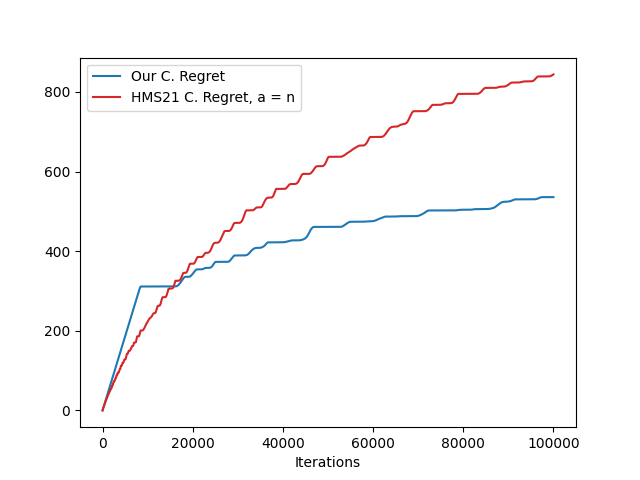} &
        \includegraphics[width=4.3cm]{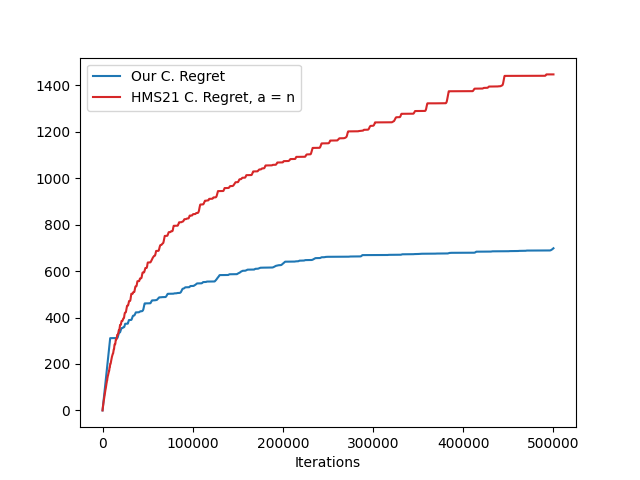} \\
        \includegraphics[width=4.3cm]{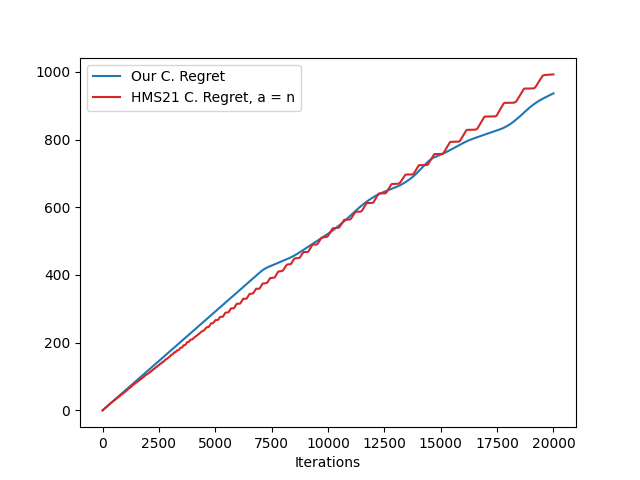} &
        \includegraphics[width=4.3cm]{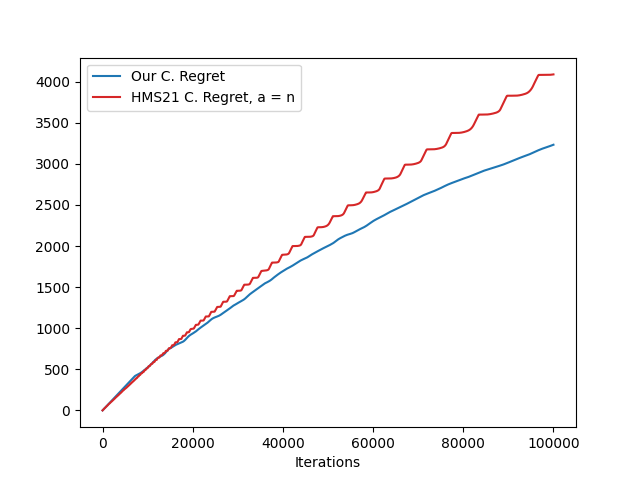} &
        \includegraphics[width=4.3cm]{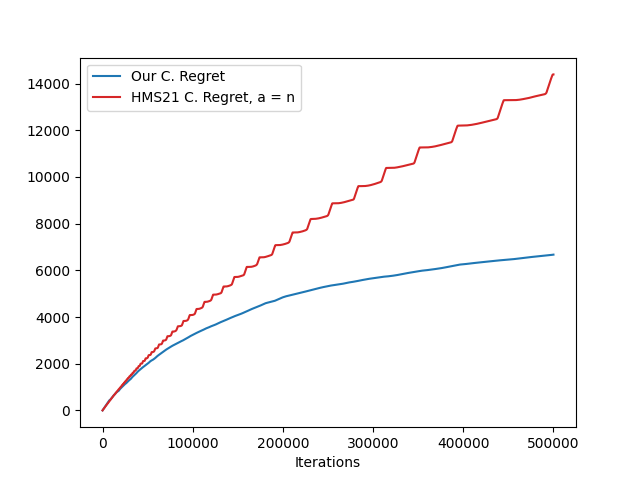} \\
        \includegraphics[width=4.3cm]{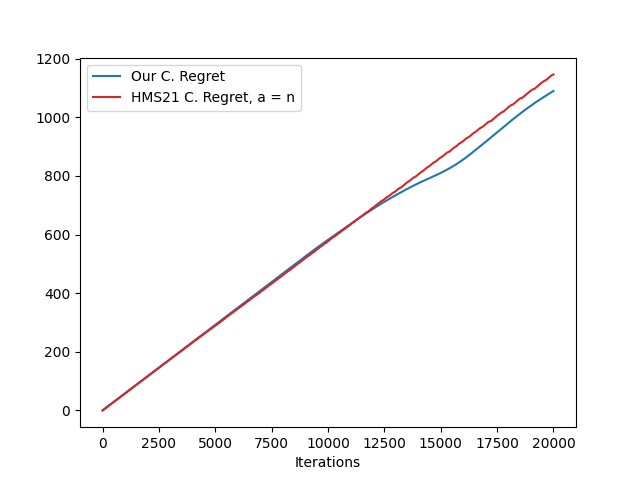} &
        \includegraphics[width=4.3cm]{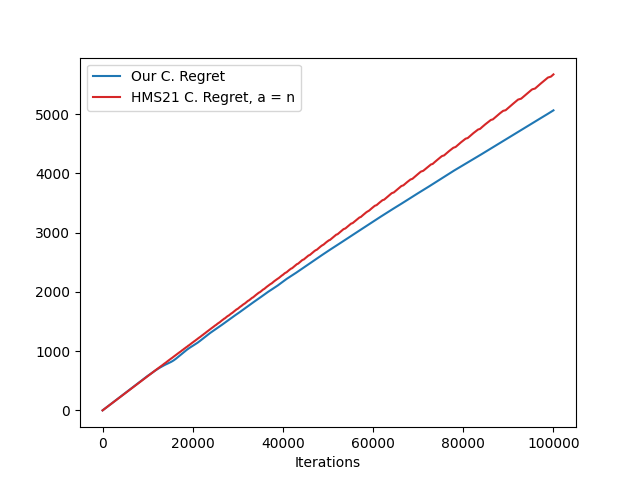} &
        \includegraphics[width=4.3cm]{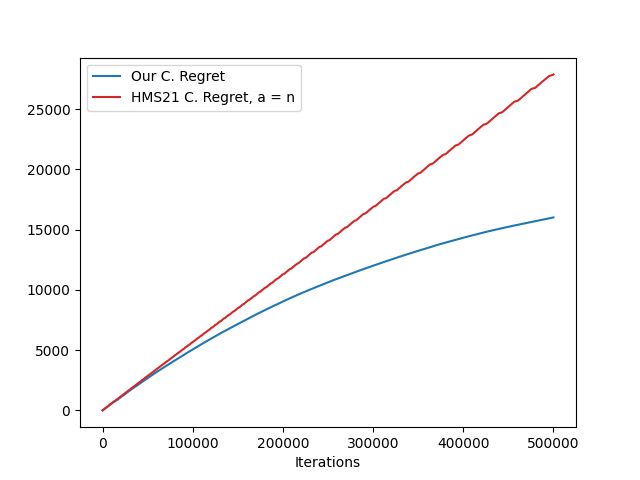} \\
    \end{tabular}
     \caption{Sample Cumulative Regret Graphs for each setting. $(\np, \na)$ is (4,2) in the top row, (20,4) in the center row, and (80,8) on the bottom row. The regret graphs show up to 20,000 iterations in the first column, 100,000 iterations in the second column, and 500,000 iterations in the third.}
    \label{fig:imgs1}
\end{figure*}

\begin{table*}[ht]
\caption{Average cumulative regrets over 10 values of $\mu^*$, with $T = 5 \cdot 10^5$}
\centering
\begin{tabular}{cc|cc|cc|c}
$\np$ & $\na$ & \multicolumn{2}{c|}{Algorithm \ref{alg:fair_UCB}} & \multicolumn{2}{c|}{\cite{hossain2020fair}} & $\F(\po^\star, \mu^\star )$ \\ \hline
   &   & $t = 2 \cdot 10^5 $ & $t = 5\cdot 10^5 $ & $t = 2 \cdot 10^5$ & $t = 5\cdot 10^5 $ &                     \\
4  & 2 & $ 1222 $      & $ 1806 $      & $ 1226 $      & $ 1832 $      & $0.9428 \pm 0.0370$ \\
20 & 4 & $ 8313 $      & $ 15322 $     & $ 10801 $     & $21655 $      & $0.6353 \pm 0.0358$ \\
80 & 8 & $ 4521 $      & $ 9966 $      & $ 5230 $      & $ 12874 $     & $0.0808 \pm 0.0154$
\end{tabular}
\label{tab:exp}
\end{table*}

We test algorithm~\ref{alg:fair_UCB} and the UCB algorithm from \cite{hossain2020fair} with both $\alpha_t = N$ and $\alpha_t = \sqrt{12NK\log(NKt)}$, although we exclude the second $\alpha_t$ from the results because it was outperformed by the other algorithms in every test. We test three pairs of $(\np, \na)$: a small size $(4,2)$, a medium size $(20,4)$, and a large size $(80,8)$. For each pair $(\np, \na)$ we test the algorithms on 5 values of $\mu^\star$ chosen randomly from 1 minus an exponential distribution with mean $0.04$, and rounded up to 0.1 in extreme cases. When an action $a \in [\na]$ is taken, we draw the rewards for each agent $j \in [\np]$ from a Bernoulli distribution with $p = \mu_{j,a}^\star$. For the optimization step, we round  the empirical mean up to $10^{-3}$ since this is a divisor in gradient computations. We also optimized the additive terms in both algorithms using a constant factor found through empirical binary search: the additive term $w_{j,a,t}$ in Algorithm \ref{alg:fair_UCB} is scaled by 0.5, and the additive term in the optimization step of \cite{hossain2020fair} is scaled by 0.8.

In each algorithm we compute $\pi_t$ using the projected gradient ascent. For algorithm \ref{alg:fair_UCB}, we take advantage of the log-concavity of the Nash Social Welfare function and the monotonicity of the logarithm and optimize the log of the objective in the gradient ascent.. 
%Note that we perform optimizations for the protocols $\pi_t$ using projected gradient ascent, which is sufficient for our algorithm since the NSW is a log-concave objective.
The UCB algorithm from \cite{hossain2020fair} computes the policy $\pi_t$ as $\arg\max_{\pi} \F(\pi, \hat{\mu}_t) + \alpha_t\sum_{a\in[K]}\left(\pi_a \cdot \sqrt{\frac{\log(NKt)}{N_{a,t}}}\right)$, which is no longer a log-concave objective due to the linear terms
 \cite{hossain2020fair}. We can address this issue to some degree by allowing the gradient ascent to run substantially longer: for algorithm \ref{alg:fair_UCB} the gradient ascent terminates after the objective changes by less than $2\cdot10^{-4}$ after 20 iterations, and for \cite{hossain2020fair} the ascent terminates after changing less than $10^{-6}$ after 30 iterations. By tightening the termination conditions we make it more difficult for the ascent to hang on a suboptimal position in \cite{hossain2020fair} at the tradeoff of longer runtime. Even at this point, we still see instability in the regret graph over time with the non-concave optimization. We include the algorithm none-the-less as it is the only other existing algorithm for fair multi-agent multi-armed bandits with regret on the order of $O(\sqrt{T})$.
 
 %and therefore is the best point of comparison for the performance of our algorithm.

% \begin{figure*}[ht]
%     \centering
%     \begin{tabular}{c c c}
%         \includegraphics[width=4.5cm]{images/results_4_2_3.png} & \includegraphics[width=4.5cm]{images/results_20_4_0.png} & \includegraphics[width=4.5cm]{images/results_80_8_1.png} \\
%         \includegraphics[width=4.5cm]{images/results_4_2_4.png} & \includegraphics[width=4.5cm]{images/results_20_4_3.png} & \includegraphics[width=4.5cm]{images/results_80_8_2.png}
%     \end{tabular}
%     \caption{Sample Cumulative Regret Graphs, on 2 runs with distinct values of $\mu^*$ for each setting. $(\np, \na)$ is (4,2) in the left column, (20,4) in the center, and (80,8) on the right.}
%     \label{fig:imgs1}
% \end{figure*}

% \begin{table}[ht]
% \caption{Cumulative Regrets, averages over 10 values of $\mu^*$}
% \centering
% \begin{tabular}{|c|c||c|c||c|c||c|}
%     $\np$ & $\na$ & \multicolumn{2}{c||}{Algorithm \ref{alg:fair_UCB}} & \multicolumn{2}{c||}{\cite{hossain2020fair}} & $\F^*$  \\
%     & & $T = 200,000$ & $T = 500,000$ & $T = 200,000$ & $T = 500,000$ & \\
%     \hline
%     4 & 2 & $ 1222 $ & $ 1806 $ & $ 1226 $ & $ 1832 $ & $0.9428 \pm 0.0370$ \\ \hline
%     20 & 4 & $ 8313 $ & $ 15322 $ & $ 10801 $ & $21655 $ & $0.6353 \pm 0.0358$ \\  \hline
%     80 & 8 & $ 4521 $ & $ 9966 $ & $ 5230 $ & $ 12874 $ & $0.0808 \pm 0.0154$ \\\hline
% \end{tabular}
% \label{tab:exp}
% \end{table}

Table \ref{tab:exp} shows the average regret of the two algorithms after 200,000 iterations and 500,000 iterations over the 5 independent instances for each size, which are the same set of instances for both algorithms. Figure \ref{fig:imgs1} shows cumulative
regret graphs of the algorithms' instances for each setting of $\np$ and $\na$.

In all cases, algorithm \ref{alg:fair_UCB} outperforms the previous best algorithm as $\ts$ becomes large. Once the $\sqrt{T}$ terms in the regret bounds become dominant, the $\sqrt{K}\cdot\sqrt{\min{N,K}}$ factor saved in algorithm \ref{alg:fair_UCB} over the previous algorithm becomes apparent. This is especially true as $N$ and $K$ are larger. For the small case where $(N,K)=(4,2)$, both regret curves take on the shape $\sqrt{T}$, with a growing separation between the two as $T$ increases. However, it is worth noting that \cite{hossain2020fair} does outperform at early iterations. This is due to the difference in the two UCB approaches. Specifically, \cite{hossain2020fair} does not make any changes to $\hat{\pa}$ in the optimization step, so their algorithm is able to begin improving immediately. Algorithm \ref{alg:fair_UCB} adds the confidence bound to $\hat{\pa}$ and then caps all elements in $\hat{\pa}$ at 1, so until one of the terms in the upper confidence bound drops below 1 the algorithm will choose a uniform $\po$, which accounts for the large linear cumulative regret in the first 7500 iterations before the upper confidence bound is non-trivial and algorithm \ref{alg:fair_UCB} begins to outperform \cite{hossain2020fair}. For the small size, it seems that algorithm \ref{alg:fair_UCB} performs better as long as the number of iterations is at least 25,000. 

 In the medium case, the linear section in the regret curve of algorithm \ref{alg:fair_UCB} still does not outperform \cite{hossain2020fair}, but \cite{hossain2020fair} sees a steeper regret curve which narrows the gap for small $\ts$ and creates a larger regret gap for large $\ts$. In the largest case, where $(\np, \na)=(80,8)$, we see that \cite{hossain2020fair} barely ever outperforms algorithm \ref{alg:fair_UCB} even at small $\ts$ when algorithm \ref{alg:fair_UCB}'s upper confidence bound is a matrix of 1s. At high values of $\ts$ we still see that algorithm \ref{alg:fair_UCB} outperform and we observe that the other factors in the $\sqrt{\ts}$ terms begin to play a significant role. Algorithm \ref{alg:fair_UCB}'s regret curve still takes on a $\sqrt{\ts}$ shape although it is much gentler than the smaller cases, which can be at least partially attributed to the fact that the instantaneous regret at each round is bounded by the optimal value of the Nash social welfare, which is the product of expected rewards. The algorithm from \cite{hossain2020fair} still appears almost linear even at 500,000 iterations, as substantially larger values of $\ts$ are required to overcome the factors of $\np$ and $\na$ and see the $\sqrt{\ts}$ shape. There is still a substantial difference between the cumulative regrets of the two algorithms as $\ts$ increases.

Our experiments support the theoretical gains of our results. We see that at small values of $\ts$, under 10,000, our algorithm may be outperformed by \cite{hossain2020fair} for small sizes of $\np$ and $\na$. This effect becomes significantly weaker as $\np$ and $\na$ increase. At sufficiently large $\ts$, on the order of 25,000, algorithm~\ref{alg:fair_UCB} outperforms the previous state-of-the-art with increasing significance and consistency as $\na$, $\np$, and $\ts$ increase. Additionally, as $\ts$ increases the regret curves of algorithm~\ref{alg:fair_UCB} are significantly smoother than those of the previous algorithm due to the efficiency of the optimization step.

%% file: symbols.tex
% \nomenclature{\(\np\)}{Number of agents}
% \nomenclature{\(h\)}{Planck constant}
% \printnomenclature
\section{Notations}
\begin{table}[htp!]
\begin{tabular}{|c|l}
$\np$        & Number of agents                      \\
$\na$        & Number of arms                        \\
$\ts$        & Time horizon                          \\
$\po_t$      & Arm selection policy at round $t$     \\
$a_t$        & Arm selected at round $t$             \\
$\st_{a,t}$ & Number of times that arm $a$ has been selected up to round $t$                  \\
$\tpa_{j,a}$ & Mean reward of arm $a$ from agent $j$ \\
$\epa_{j,a,t}$                & Empirical mean reward of arm $a$ from agent $j$ up to round $t$                 \\
$w_{j,a,t}$                   & Confidence interval of mean reward of arm $a$ from agent $j$ up to round $t$    \\
$\UCB_{j,a,t}$                & Upper confidence bound of mean reward of arm $a$ from agent $j$ up to round $t$ 
\end{tabular}
\caption{Table of notations.}
\label{table:notations}
\end{table}

%% file: appendix_post.tex
\section{Missing Proofs}
We first  analyze the Lipschitz-continuity of $\F(\po,\mu)$ when the policy $\po$ is fixed.
\begin{lemma}[Lemma~\ref{lemma:lipschitz_main}]
\label{lemma:lipschitz}
Given a policy $\po \in \Delta^k$ and reward matrices $\mu^1, \mu^2 \in [0,1]^{\np \times \na}$, we have  
\[
\left\lvert\F(\po,\mu^1) - \F(\po,\mu^2)\right\rvert \leq \sum_{j \in [\np]} \sum_{a \in [\na]} \po_{a}  \left\lvert \mu^1_{j,a} - \mu^2_{j,a} \right\rvert
\]
\end{lemma}
\begin{proof}
We express the difference as the telescoping sum, 
\begin{align*}
\F(\po,\mu^1) - \F(\po,\mu^2)  & = \prod_{j \in [\np]} \sum_{a \in [\na]} \po_{a}  \mu^1_{j,a}  - \prod_{j \in [\np]} \sum_{a \in [\na]} \po_{a}  \mu^2_{j,a} \\ 
& = \sum_{j \in [\np]} \sum_{a \in [\na]} \po_{a} (\mu^1_{j,a} -  \mu^2_{j,a}) \prod_{j' = 1}^{j-1}\sum_{a' \in [\na]} \po_{a'} \mu^1_{j',a} \prod_{j'' = j+1 }^{\np}  \sum_{a'' \in [\na]} \po_{a} \mu^2_{j'',a} \\ 
& \le  \sum_{j \in [\np]} \sum_{a \in [\na]} \po_{a} \left\lvert \mu^1_{j,a} -  \mu^2_{j,a} \right\rvert
\end{align*}
The inequality is due to the fact that $\sum_{a \in [\na]} \po_{a} \mu_{j,a} \leq 1$. Similarly, we have, 
\begin{align*}
    \F(\po,\mu^2)  - \F(\po,\mu^1) \leq  \sum_{j \in [\np]} \sum_{a \in [\na]} \po_{a} \left\lvert \mu^1_{j,a} -  \mu^2_{j,a} \right\rvert
\end{align*}
\end{proof}

Next, we derive a variant of the Chernoff bound. 
\begin{lemma}
\label{lemma:chernoff}
Suppose $X_1 \ldots X_m$ are independent random variable taking values in $[0,1]$. Let $X =\sum_{i=1}^m X_i , \mu = \E[X]$, with probability  $1-\delta$, 
\[
\left\lvert X - \mu \right\rvert  \leq \sqrt{3 \mu\ln(2/\delta)} + 3  \ln(2/\delta) 
\]
\end{lemma}
\begin{proof}
Recall the  multiplicative Chernoff bound, 
\[
\Pr [\left\lvert X - \mu \right\rvert \ge \alpha \mu ] \leq  2 \exp(- \min( \alpha^2,  \alpha) \mu /3) 
\]
If $\mu \leq 3 \ln(2/\delta)$, set $\alpha = 3 \ln (2/\delta) / \mu $, we have:
\[
\Pr [\left\lvert X - \mu \right\rvert \ge  3  \ln(2/\delta)  ] \leq \exp(- \alpha \mu /3) = \delta
\]
Else, set $\alpha =  \sqrt{3 \ln (2/\delta) / \mu}$,
\[
\Pr [\left\lvert X - \mu \right\rvert \ge \sqrt{3 \mu\ln(2/\delta)}  ] \le  \exp(- \alpha^2 \mu /3) = \delta 
\]
\end{proof}

\begin{lemma}[Lemma~\ref{lemma:LPfair_UCB_param_in_CR_main}]
\label{lemma:LPfair_UCB_param_in_CR}

For any $\delta\in(0,1)$, with probability at least $1-\delta/2$, $\forall$ $t>\na,a\in[\na]$, $j\in[\np]$, $\left\lvert\pa^{\star}_{j, a} - \hat{\pa}_{j, a, t} \right\rvert \leq  \sqrt{\frac{12 (1 - \hat{\pa}_{j, a, t}) \logterm}{\st_{a,t}}}  +  \frac{12  \logterm}{\st_{a,t}} = w_{j,a,t}$. 
\end{lemma}
\begin{proof}

By lemma~\ref{lemma:chernoff}, we have with probability at least $1-\delta/(2\np\na\ts)$, :
\[
\left\lvert(1 - \hat{\pa}_{j, a, t} ) - (1 - \pa^{\star}_{j, a}) \right\rvert \leq \sqrt{\frac{3 (1-\pa^{\star}_{j, a}) \logterm}{\st_{a,t}}} + \frac{3  \logterm}{\st_{a,t}}. 
\] 
We also have: 
\begin{align*}
     & (1-\pa^{\star}_{j, a}) - (1 - \hat{\pa}_{j, a, t}) \leq \sqrt{\frac{3 (1-\pa^{\star}_{j, a}) \logterm}{\st_{a,t}}} + \frac{3  \logterm}{\st_{a,t}} \\
\end{align*}
If $(1-\pa^{\star}_{j, a})  - \sqrt{\frac{3 (1-\pa^{\star}_{j, a}) \logterm}{\st_{a,t}}} \geq 0 $:
\[
\sqrt { 1-\pa^{\star}_{j, a}  - \sqrt{\frac{3 (1-\pa^{\star}_{j, a}) \logterm}{\st_{a,t}}}}  \geq \sqrt{1-\pa^{\star}_{j, a} } - \sqrt{\frac{3 \logterm}{\st_{a,t}}}
\]
In either cases, we have: 
\begin{align*}
          & \sqrt{1 - \hat{\pa}_{j, a, t}+ \frac{3  \logterm}{\st_{a,t}} }\geq \sqrt{1-\pa^{\star}_{j, a} } - \sqrt{\frac{3 \logterm}{\st_{a,t}}} \\
          \implies & \sqrt{1 - \hat{\pa}_{j, a, t}+ \frac{3  \logterm}{\st_{a,t}} } + \sqrt{\frac{3 \logterm}{\st_{a,t}}} \geq \sqrt{1-\pa^{\star}_{j, a} } \\
          \implies &2 \sqrt{1 - \hat{\pa}_{j, a, t}+ \frac{3  \logterm}{\st_{a,t}} } \cdot \sqrt{\frac{3 \logterm}{\st_{a,t}}}  +  \frac{6  \logterm}{\st_{a,t}}  \geq  \hat{\pa}_{j, a, t} - \pa^{\star}_{j, a}   \\
        \implies & \sqrt{\frac{12 (1 - \hat{\pa}_{j, a, t}) \logterm}{\st_{a,t}}}  +  \frac{12  \logterm}{\st_{a,t}}  \geq  \hat{\pa}_{j, a, t} - \pa^{\star}_{j, a}  \\
    %  \implies & 1 - \hat{\pa}_{j, a, t}+ \frac{3  \ln(4\na\ts/\delta)}{\st_{a,t}} \geq  1-\pa^{\star}_{j, a} - \sqrt{\frac{3 1-\pa^{\star}_{j, a} \ln(4\na\ts/\delta)}{\st_{a,t}}} \geq \left( \sqrt{1-\pa^{\star}_{j, a} } - \sqrt{\frac{3 \ln(4\na\ts/\delta)}{\st_{a,t}}} \right)^2 
\end{align*}
Similarly for the other direction, we have: 
\[
\left\lvert \hat{\pa}_{j, a, t}-\pa^{\star}_{j, a}  \right\rvert \leq  \sqrt{\frac{12 (1 - \hat{\pa}_{j, a, t}) \logterm}{\st_{a,t}}}  +  \frac{12  \logterm}{\st_{a,t}} 
\]
The lemma follows by applying the union bound. 
\end{proof}

\begin{lemma}[Lemma~\ref{lemma:highreward_main}]
\label{lemma:highreward}
Define  $g_{j,t} = \sum_{a\in [\na]} \po_{a, t} \left(1-\UCB_{j,a,t}\right)$, and $S(t,p) = \left\{j \text{ for } j \in [\np] :  g_{j,t} \geq  2^{-p} \right\}.$  If  $\left\lvert S(t,p) \right\rvert < 2^p \cdot 3\ln \ts $ for all $p \geq 0 $, then 
 %\[\sum_{j \in [\np]} \sum_{a \in [\na]}  \po_{a, t} (1-\hat{\pa}_{j,a,t}) \leq  3\ln\ts(1+2\log n) + \sum_{j \in [\np]} \sum_{a \in [\na]}   \po_{a,t} w_{j,a,t}   \]
 \begin{align*}
     &\sum_{j \in [\np]} \sum_{a \in [\na]}  \po_{a, t} (1-\hat{\pa}_{j,a,t}) \\  &\leq 1 + 6\ln \ts \log \np + \sum_{j \in [\np]} \sum_{a \in [\na]}   \po_{a,t} w_{j,a,t}.
 \end{align*}
\end{lemma}

\begin{proof}
Define $S'(t,p) = S(t, p+1) \setminus  S(t, p)$. Note that $\sum_{j \in S'(t, p)}  g_{j,t} \leq |S(t,p+1)|\cdot 2^{-p}< 2^{p+1} \cdot 3 \ln \ts \cdot 2^{-p}  $. We have:
\begin{align*}
   \sum_{j \in [\np]} g_{j,t} &\leq \sum_{j \notin S(t,\log \np)} g_{j,t}  + \sum_{j' \in S(t,\log \np)}  g_{j',t}  \\ 
   & \leq  \sum_{j \notin S(t,\log \np)} g_{j,t}  + \sum_{j' \in S(t,\lfloor \log \np \rfloor)}  g_{j',t}  \\
    & \leq \np\cdot\frac{1}{\np} + \sum_{p=0}^{\lfloor  \log \np \rfloor} \sum_{j' \in S'(t, p)}  g_{j',t} \\
   % & \leq  1 +  \sum_{p=0}^{\log \np} \sum_{j' \in S(t, p+1)}  2 \cdot g_{j',t} \\
    & \leq \np\cdot\frac{1}{\np} + \sum_{p=0}^{\lfloor  \log \np \rfloor} 2^{p+1} \cdot 3 \ln \ts \cdot 2^{-p} \\ & \leq 1 + 6\ln \ts \log \np.   \numberthis \label{exposition0} 
\end{align*}

Thus, 
\begin{align*}
    &\sum_{j \in [\np]} \sum_{a \in [\na]}  \po_{a, t} (1-\hat{\pa}_{j,a,t}) - 1 -6\ln \ts \log \np \\ 
    &  \leq \sum_{j \in [\np]} \sum_{a \in [\na]}  \po_{a, t} (1-\hat{\pa}_{j,a,t}) - \sum_{j \in [\np]} g_{j,t}\\
        &  = \sum_{j \in [\np]} \sum_{a \in [\na]}  \po_{a, t} (1-\hat{\pa}_{j,a,t}) - \po_{a, t} \left(1-\UCB_{j,a,t}\right) \\
        & = \sum_{j \in [\np]} \sum_{a \in [\na]}  \po_{a, t} (\UCB_{j,a,t} -\hat{\pa}_{j,a,t}) \\
    & \leq \sum_{j \in [\np]} \sum_{a \in [\na]}   \po_{a,t} w_{j,a,t}. 
\end{align*}
The first inequality is due to equation~\ref{exposition0}. The last inequality is due to the definition of $\UCB_{j,a,t}$.
\end{proof}

\begin{lemma}[Lemma~\ref{lemma:LPfair_UCB_concentration_width_nonsqr_main}]
\label{lemma:LPfair_UCB_concentration_width_nonsqr}
With probability $1-\delta/2, \sum_{t \in [\ts]} \sum_{a \in [\na]} \po_{a,t} / \st_{a,t} \leq  2 \na\left(\ln\frac{\ts}{\na} + 1  \right)+\ln (2/\delta).$
\end{lemma} 
\begin{proof}
Let $v_{t} = \sum_{a \in [\na]} \po_{a,t} / \st_{a,t}$. Consider the $\sigma$-algebra $\f_t = (a_1, a_2 \ldots a_t) $. Note that $v_{a,t}$ is $\f_t$-measurable. Define $g(y)=2\sum_{i=2}^{\infty}\frac{y^{i-2}}{i!}=2\frac{e^{y}-1-y}{y^{2}}$.  We have
\begin{align*}
  & \E_{a \sim \po_t}  [ \exp(z v_{t}) |\f_{t-1} ] = \E_{a \sim \po_t} \left[ \sum_{i=0}^{\infty}\frac{z^{i}v_t^i}{i!} |\f_{t-1} \right] \\
  & =  1+\E_{a \sim \po_t}\left[z v_{t}|\f_{t-1}\right]+\E\left[\sum_{i=2}^{\infty}\frac{\left(z v_{t}\right)^{i}}{i!}|\f_{t-1}\right] \\
  & = 1+\E_{a \sim \po_t}\left[z v_{t}|\f_{t-1}\right]+\E\left[\frac{z^{2}v_{t}^{2}}{2}g(z v_{t})|\f_{t-1}\right] \\
  & \leq 1+\E_{a \sim \po_t}\left[z v_{t}|\f_{t-1}\right]+\E\left[\frac{z^{2}v_{t}}{2}g(z)|\f_{t-1}\right] \\
  & =1+\E_{a \sim \po_t}\left[\left(z+\frac{z^{2}g(z)}{2}\right)v_{t}|\f_{t-1}\right]\\
  & =1+\E_{a \sim \po_t}\left[\left(z+\frac{z^{2}g(z)}{2}\right)\frac{1}{N_{a,t}}|\f_{t-1}\right]\\
  &\le\exp\left(\E_{a \sim \po_t}\left[\left(z+\frac{z^{2}g(z)}{2}\right)\frac{1}{N_{a,t}}|\f_{t-1}\right]\right) \\
  & \le\E_{a \sim \po_t}\left[\exp\left(\left(z+\frac{z^{2}g(z)}{2}\right)\frac{1}{N_{a,t}} \right)|\f_{t-1}\right]
\end{align*}

The first inequality is due to $v_{t} \leq  1$ and  $g(y)$ is monotonically increasing for $y\geq 0$. The last two inequalities follow from the fact that $1+x \le e^x$ and $e^{\E[x]}\le\E[e^{x}]~\forall x$. Thus by induction over $t$,
\begin{align*}
    \E_{a_1,\ldots,a_{\ts}}\left[\exp\left(z\sum_{t\in[\ts]}v_t\right)\right]&\le \E_{_{a_1,\ldots,a_{\ts}}}\left[\exp\left(\left(z+\frac{z^{2}g(z)}{2}\right)\sum_{t}\frac{1}{\st_{a_t,t}}\right)\right]  \\ & \leq \exp\left(\left(z+\frac{z^{2}g(z)}{2}\right) \na\left(\ln\frac{\ts}{\na} + 1  \right)\right)
\end{align*}
By Markov's inequality, 
\begin{align*}
    \Pr\left[\sum_{t\in[\ts]} v_t \geq \frac{\ln \alpha}{z} \right] &= \Pr\left[\exp\left(z\sum_{t\in[\ts]} v_t\right) \geq \alpha \right] \\ & \leq \exp\left(\left(z+\frac{z^{2}g(z)}{2}\right) \na\left(\ln\frac{\ts}{\na} + 1  \right)\right) / \alpha
\end{align*}
The lemma follows by setting $z = 1, \alpha = \exp\left(\left(z+\frac{z^{2}g(z)}{2}\right) \na\left(\ln\frac{\ts}{\na} + 1  \right) + \ln (2/\delta) \right)$. 
\end{proof}

\subsection{The Inefficient Algorithm with Improved Regret}

We begin by reviewing the algorithm. We first pull each arm a total of $432N^2 \ln(6NTK/\delta)$ times to establish $\hat{\mu}$ with high accuracy. Then, we select the policy at each time step in order to optimize the Nash social welfare plus the dot product with an error vector $\eta$. Note that the vector $\eta$ defined in theorem \ref{theorem:alg3_regret_app} is slightly different than in theorem \ref{theorem:alg3_regret}. This error vector matches the proof below, and yields a simpler proof. Additionally, notice that we constrain $\pi_t$ to the region $S_\pi$ in algorithm \ref{alg:fair_UCB_inefficient_app}, which also eases the proof and is a linear constraint, and therefore a convex set, at each iteration.

\begin{algorithm}[ht]
\begin{algorithmic}[1]
\STATE {\bf input: }{$\na,\np,\ts,\delta$}
\FOR{$t=1$ to $\ts$}
\IF{$t \leq 180N^2K \ln(6NTK/\delta)\ln T$} 
\STATE $\po_t \leftarrow$ policy that puts probability 1 on arm $\lceil t/(180N^2 \ln(6NTK/\delta)\ln T)\rceil$ 
\ELSE 
%\STATE $\UCB_t = \min(\epa_t + w_t,1)$
\STATE $\forall j,a, \epa_{j,a,t} = \frac{1}{\st_{a,t}} \sum_{\tau=1}^{t-1} r_{j, a_\tau, \tau}  \Indicator\{a_\tau=a\}$  %\quad \quad \# $\operatorname{Update} \epa_t$
%\STATE $\forall j,a,  w_{j,a,t} = \sqrt{\frac{12 (1 - \hat{\pa}_{j, a, t})  \logterm}{\st_{a,t}}}  +  \frac{12  \logterm}{\st_{a,t}}$  %\quad \quad \# $\operatorname{Update} w_t$
\STATE Let $S_\pi = \Delta^\na \cap \{\pi : \sum_{a\in[K]}\pi_a\left(\sum_{j\in[N]}(1-\hat{\mu}_{j,a,t})\right) \le 1 + 2\ln T\}$
\IF{$S_\pi \ne \emptyset$}
\STATE $\po_t \leftarrow \operatorname{argmax}_{\po \in S_\pi} \left(\F(\po, \epa_t) + \po\cdot \eta_{t}\right)  $ %= \operatorname{argmax}_{\po \in \Delta^\na} \log \F(\po, \UCB_t)$
\ELSE
\STATE $\pi_t \leftarrow$ policy that puts probability 1 on a random arm $a \in [K]$
\ENDIF
\ENDIF
\STATE Sample $a_t$ from $\po_t$
\STATE Observe rewards $\{r_{j,a_t,t}\}_{j \in \np}$
\STATE $\st_{a_t, t+1} \leftarrow \st_{a_t, t} + 1$
%\STATE Update $\epa_{t+1}$ and $w_{t+1}$
\ENDFOR
\end{algorithmic}
\caption{Fair multi-agent UCB algorithm with high start-up cost}
\label{alg:fair_UCB_inefficient_app}
\end{algorithm}

\begin{theorem}[Theorem~\ref{theorem:alg3_regret}]\label{theorem:alg3_regret_app}
 Suppose $\forall j, a, t,$  $\Reward_{j, t,a}\in[0,1]$, $w_{j,a,t} =  \sqrt{\frac{12 (1 - \hat{\pa}_{j, a, t}) \ln(6NKT/\delta)}{\st_{a,t}}}  +  \frac{12  \ln(6NKT/\delta)}{\st_{a,t}} $, and \begin{align*}
% \eta_{a,t} =&  \left(4\sqrt{\ln(6KT/\delta)} + 6\sqrt{2}\ln(6NKT/\delta)\right)\sqrt{\frac{K}{T}}\sum_{j\in[N]}(1-\hat{\mu}_{j,a,t}) \\
%     &+\left(\left(4\sqrt{\ln(6KT/\delta)} + \frac{1}{\sqrt{2}} \right)\sqrt{\frac{T}{K}} + 12\sqrt{2}\sqrt{N}\ln(6NKT/\delta)\right)\frac{1}{N_{a,t}}\\
%   &+\frac{1}{\sqrt{N}-1}\sum_{j\in[N]}w_{j,a,t}\\
  \eta_{a,t} =& \left(4\sqrt{\ln(6KT/\delta)} + 6\sqrt{2}\ln(6NKT/\delta)\sqrt{2+2\ln T}\right)\sqrt{\frac{K}{T}}\sum_{j\in[N]}(1-\hat{\mu}_{j,a,t}) \\
    &+\Big(\left(4\sqrt{\ln(6KT/\delta)} + \sqrt{1+ \ln T} \right)\sqrt{\frac{T}{K}} \\&+ 12\sqrt{2}\sqrt{N}\ln(6NKT/\delta)\sqrt{2+2\ln T }\Big)\frac{1}{N_{a,t}}\\
  &+\frac{1}{20\sqrt{N}/19-1}\sum_{j\in[N]}w_{j,a,t}\\
\end{align*} for any $\delta\in(0,1)$, the  regret of Algorithm \ref{alg:fair_UCB_inefficient_app} is $\TotalRegret_\ts =\widetilde{O}\left( \sqrt{\na\ts} + \np^2\na\right)$ with probability at least $1-\delta$. 
\end{theorem}

We will use several lemmas from earlier to prove this, but we will need to condition one more event in lemma \ref{lem:error_sum_bound}. Therefore, we will need to adapt lemmas \ref{lemma:LPfair_UCB_param_in_CR_main} and \ref{lemma:LPfair_UCB_concentration_width_nonsqr_main} such that they have failure probability $\delta/3$ instead of $\delta/2$.

In order to prove this theorem we introduce the following notation:
\begin{align*}
a_j(\pi) & =\sum_{a\in[K]}\pi_{a}\mu_{j,a}^{*}=1-\sum_{a\in[K]}\pi_{a}(1-\mu_{j,a}^{*})\\
 b_{j,t}(\pi) & =\sum_{a\in[K]}\pi_{a}(\hat{\mu}_{j,a,t}-\mu_{j,a}^{*}) %\\ 
% a'_j(\pi) & =\sum_{a\in[K]}\pi_{a}\hat{\mu}_{j,a}=1-\sum_{a=1}^{K}\pi_{a}(1-\hat{\mu}_{j,a})\\
% b'_{j,t}(\pi) & =\sum_{a\in[K]}\pi_{a}(\mu_{j,a}^{*}- \hat{\mu}_{j,a,t}) 
\end{align*}
 where $a_j(\pi)$ is the true expected reward of agent $j$ under policy $\pi$ and $b_{j,t}(\pi)$ is the expected error $\hat{\mu}_t-\mu^*$ for agent $j$ under policy $\pi$%, and $a'_j(\pi)$ and $b'_{j,t}(\pi)$ are defined similarly for the empirical reward.
.
Note that $\text{NSW}(\pi,\mu^{*})=\text{NSW}'(\pi,\mu^{*})=\prod_{j=1}^{N}a_{j}(\pi)$
and $\text{NSW}(\pi,\hat{\mu}_{t})=\text{NSW}'(\pi,\hat{\mu}_{t})=\prod_{j=1}^{N}(a_{j}(\pi)+b_{j,t}(\pi))$
.
We can expand the product $\text{NSW}'(\pi,\hat{\mu})$ to obtain
\begin{align*}
\text{NSW}'(\pi,\hat{\mu}_t) & =\prod_{j=1}^{N}(a_{j}(\pi)+b_{j,t}(\pi))\\
 & =\left(\prod_{j=1}^{N}a_{j}(\pi)\right)\left(1+\sum_{m=1}^{N}\sum_{B\in\{S\subseteq[N]:|S|=m\}}\prod_{i\in B}\frac{b_{j,t}(\pi)}{a_{j}(\pi)}\right)\\
\implies
\text{\text{NSW}(\ensuremath{\pi,}\ensuremath{\hat{\mu}_t})}-{\text{NSW}(\ensuremath{\pi},\ensuremath{\mu^{*}})} & =\text{\text{NSW}'(\ensuremath{\pi,}\ensuremath{\hat{\mu}_t})}-{\text{NSW}'(\ensuremath{\pi},\ensuremath{\mu^{*}})}\\
 & =\left(\prod_{j=1}^{N}a_{j}(\pi)\right)\left(\sum_{m=1}^{N}\sum_{B\in\{S\subseteq[N]:|S|=m\}}\prod_{j\in B}\frac{b_{j,t}(\pi)}{a_{j}(\pi)}\right).
\end{align*}

% We also have: 

% \begin{align*}
% \text{NSW}'(\pi,\mu^{*}_t) & =\prod_{j=1}^{N}(a'_{j}(\pi)+b'_{j,t}(\pi))\\
%  & =\left(\prod_{j=1}^{N}a'_{j}(\pi)\right)\left(1+\sum_{m=1}^{N}\sum_{B\in\{S\subseteq[N]:|S|=m\}}\prod_{i\in B}\frac{b'_{j,t}(\pi)}{a'_{j}(\pi)}\right)\\
% \implies
% {\text{NSW}(\ensuremath{\pi},\ensuremath{\mu^{*}})} - \text{\text{NSW}(\ensuremath{\pi,}\ensuremath{\hat{\mu}_t})} & ={\text{NSW}'(\ensuremath{\pi},\ensuremath{\mu^{*}})} - \text{\text{NSW}'(\ensuremath{\pi,}\ensuremath{\hat{\mu}_t})}\\
%  & =\left(\prod_{j=1}^{N}a'_{j}(\pi)\right)\left(\sum_{m=1}^{N}\sum_{B\in\{S\subseteq[N]:|S|=m\}}\prod_{j\in B}\frac{b'_{j,t}(\pi)}{a'_{j}(\pi)}\right).
% \end{align*}

We will omit the $t$ from the subscript of $b$ and we will omit the arguments $\pi$ for conciseness in sections of this proof.

The proof structure will be ordered as follows. First, we prove a Chernoff bound in lemma \ref{lem:error_sum_bound} and 
bounds on $\sum_{a\in[K]}\pi_a\left(\sum_{j\in[N]}(1-\mu_{j,a})\right)$ for useful pairs of $\pi$ and $\mu$ in lemmas \ref{lemma:highreward_star}, \ref{lemma:highreward_muhat_app}, and \ref{lemma:highreward_pit_mustar}
to bound terms of the form $\sum_{j\in[N]}b_{j,t}(\pi)$. Then, we bound the terms with $m \ge 2$ using lemma \ref{lem:high_order_bound} and we bound the terms with $m = 1$ in lemmas \ref{lem:sum_bj_bound} and lemma \ref{lem:sqrt_sumsquared_bound} combined in lemma \ref{lem:low_order_bound}. With all the terms bounded, we can obtain a bound on the entire error summed over $t \in [T]$ in order to prove theorem \ref{theorem:alg3_regret_app}.

\begin{lemma}
\label{lem:error_sum_bound}
With probability at least $1-\delta/3$, for all $a\in[K]$ and all $t\in(180N^2K\ln(6NTK/\delta)\ln(T),T]$ we have \[\left|\sum_{j\in[N]}\left(\hat{\mu}_{j,a,t}-\mu^*_{j,a}\right)\right| \le \sqrt{\frac{4\ln(6KT/\delta)\sum_{j\in[N]}(1-\mu_{j,a}^{*})}{N_{a,t}}}.\]
\end{lemma}
\begin{proof}
Fix $t$ to be a value between $1$ and $T$ and $a\in [K]$.
Let $v_{j,a,t}$ be the value we obtain for agent $j$ when we pull arm $a$ for the $t$-th time. Let $x_{j,a,t}=v_{j,a,t}-\mu^*_{j,a}$. Observe that $x_{j,a,t}$ is in the range $[-1,1]$ with mean $0$ and variance $\mu^*_{j,a}(1-\mu^*_{j,a})\le 1-\mu^*_{j,a}$. Thus, $X_{t_0}=\sum_{j\in[N]} x_{j,a,t_0}$ has mean $0$ and variance at most $\sigma^2 = \sum_{j\in [N]} (1-\mu^*_{j,a})$. By the Chernoff bound (see e.g. Theorem 3.1 in \cite{chung2006concentration}),

\[
\Pr[|\sum_{t_0=1}^t X_{t_0}| \ge k\sigma] \le 2\exp(-k^2/(4t))
\]

By choosing $k = \sqrt{4t\ln(6KT/\delta)}$, we have that with probability $1-\delta/(3KT)$,

\begin{align*}
\left|\sum_{t_0=1}^t X_{t_0}\right| &\le \sqrt{4t\ln(6KT/\delta)\sum_{j\in [N]} (1-\mu^*_{j,a})}\\
\left|\frac{1}{t}\sum_{t_0=1}^t X_{t_0}\right| &\le \sqrt{\frac{4\ln(6KT/\delta)\sum_{j\in [N]} (1-\mu^*_{j,a})}{t}}
\end{align*}

By union bound over all choices of $t\in [T]$ and $a\in[K]$, the above inequality holds for all $t$ and $a$ simultaneously with probability $1-\delta/3$. Note that for all time horizon $t\in [T]$, because $N_{a,t} \in [T]$ and $\hat{\mu}_{j,a,t}$ is the average of the $N_{a,t}$ values we obtained for agent $j$ and arm $a$, it follows that with probability $1-\delta/3$, \[\left|\sum_{j\in[N]}\left(\hat{\mu}_{j,a,t}-\mu^*_{j,a}\right)\right| \le \sqrt{\frac{4\ln(6KT/\delta)\sum_{j\in[N]}(1-\mu_{j,a}^{*})}{N_{a,t}}}~\forall a,t.\]
\end{proof}

Before we begin to bound the terms in our expansion of the regret, we introduce a few lemmas in the vein of lemma \ref{lemma:highreward_main} in order to bound $\sum_{a\in[K]}\pi_a\left(\sum_{j\in[N]}(1-\mu_{j,a})\right)$ for useful pairs of $\pi$ and $\mu$. We show that these hold when conditioning on an event whose failure gives us a trivial regret bound in lemma \ref{lemma:trivial_regret_app}.

\begin{lemma}
    Given $N$ values $0 \le a_1,...,a_N \le 1$, if  $\prod_{j\in[N]}a_j \ge T^{-2}$ then $\sum_{j\in[N]}(1-a_j) \le 2 \ln T$ for all $T \ge 1$.
    \label{lemma:prod_sum_exchange}
\end{lemma}
\begin{proof}
    By the inequality $1 + x \le e^x$, we have
    \begin{align*}
        1 + (a_j - 1) &\le e^{a_j - 1}\\
        \prod_{j\in[N]}a_j &\le e^{\sum_j(a_j - 1)}\\
        \implies \ln\left( \prod_{j\in[N]}a_j \right) &\le \sum_j(a_j - 1)\\
        \implies -2\ln(T) &\le \sum_{j\in[N]}(a_j - 1)\\
        \implies 2\ln(T) &\ge \sum_{j\in[N]}(1-a_j)
    \end{align*}
\end{proof}

For any $\pi$, $\mu$ in our problem, then, we have that $$\F(\pi,\mu) = \prod_{j\in[N]}\left(\sum_{a\in[K]}\pi_a\mu_{j,a}\right) \le T^{-2}$$ or $$\sum_{j\in[N]}\left(1-\sum_{a\in[K]}\pi_a\mu_{j,a}\right) = \sum_{j\in[N]}\left(\sum_{a\in[K]}\pi_a(1-\mu_{j,a})\right) \le 2\ln T$$ which yields the immediate corollary:

\begin{lemma}
If $\F(\pi^*, \mu^*) \ge T^{-2}$ then $\sum_{j\in[N]}\left(\sum_{a\in[K]}\pi^*_a(1-\mu^*_{j,a})\right) \le 2\ln T$.
    \label{lemma:highreward_star}
\end{lemma}

Additionally, since we have that $\F(\pi_t, \mu^*) \ge 0$, we have the following trivial regret bound:
\begin{lemma}
If $\sum_{j\in[N]}\left(\sum_{a\in[K]}\pi^*_a(1-\mu^*_{j,a})\right) \ge 2\ln T$, then we obtain the total regret bound \[R_T = \sum_{t\in[T]}\left(\F(\pi^*, \mu^*) - \F(\pi_t, \mu^*)\right) \le \sum_{t\in[T]}\F(\pi^*, \mu^*) \le \frac{1}{T}.\]
\label{lemma:trivial_regret_app}
\end{lemma}

Since we achieve a trivial regret bound in the other case, we condition on the event that $\sum_{j\in[N]}\left(\sum_{a\in[K]}\pi^*_a(1-\mu^*_{j,a})\right) \le 2\ln T$ going forward.

Next, we want to achieve a similar bound with $\hat{\mu}_t$ in place of $\mu^*$:

\begin{lemma}
 If $\sum_{j\in[N]}\left(\sum_{a\in[K]}\pi^*_a(1-\mu^*_{j,a})\right) \le 2\ln T$, then
    $$\sum_{j\in[N]}\sum_{a\in[K]}\pi^*_a(1-\hat{\mu}_{j,a,t})\le 1 + 2 \ln T $$ for all $t \ge 180N^2\ln(6NKT/\delta)\ln T$.
    \label{lemma:highreward_muhat_app1}
\end{lemma}
\begin{proof}
We can achieve this bound using lemma \ref{lemma:highreward_star} and a bound on the term $\sum_{j\in[N]}\sum_{a\in[K]}\pi^*_a(\mu^*_{j,a} - \hat{\mu}_{j,a,t})$:
\begin{align*}
    \sum_{j\in[N]}\sum_{a\in[K]}\pi^*_a(1-\hat{\mu}_{j,a,t}) &= \sum_{j\in[N]}\sum_{a\in[K]}\pi^*_a(1-\mu^*_{j,a}) + \sum_{j\in[N]}\sum_{a\in[K]}\pi^*_a(\mu^*_{j,a} - \hat{\mu}_{j,a,t})\\
    &\le 2 \ln T + \sum_{a\in[K]}\pi^*_a\left(\sum_{j\in[N]}\left(\mu^*_{j,a}-\hat{\mu}_{j,a,t}\right)\right)\\
    &\le 2 \ln T + \sum_{a\in[K]}\pi^*_a\left|\sum_{j\in[N]}\left(\mu^*_{j,a}-\hat{\mu}_{j,a,t}\right)\right|\\
    &\le 2 \ln T + \sum_{a\in[K]}\pi^*_a\sqrt{\frac{4\ln(6KT/\delta)\sum_{j\in[N]}(1-\mu_{j,a}^{*})}{N_{a,t}}}
\end{align*}
where the first inequality follows from lemma~\ref{lemma:highreward_star}, the last inequality follows from conditioning on lemma \ref{lem:error_sum_bound}. With $N_{a,t} \ge 180N^2\ln(6NKT/\delta)\ln(T)$ for $t \ge 180N^2K\ln(6NKT/\delta)\ln(T)$, we have \[\sqrt{\frac{4\ln(6KT/\delta)\sum_{j\in[N]}(1-\mu_{j,a}^{*})}{N_{a,t}}} \le \sqrt{\frac{4\ln(6KT/\delta)\cdot N}{180N^2\ln(6NKT/\delta)\ln T}} \le \frac{1}{5\sqrt{N}}\]
for all $a \in [K]$. The bound holds since $\sum_{a\in[K]}\pi^*_a = 1$.
\end{proof}

We then restrict $\pi_t$ to the linear constraint $\sum_{a\in[K]}\pi_{a,t}\left(\sum_{j\in[N]}(1-\hat{\mu}_{j,a,t})\right) \le 1 + 2\ln T$, as given by the set $S_\pi$ in algorithm \ref{alg:fair_UCB_inefficient_app}. By contraposition, if no feasible solution $\pi_t$ exists then we know our conditioning has failed, and any policy $\pi$ will give us the trivial regret bound in lemma \ref{lemma:trivial_regret_app}. Otherwise, a feasible solution exists and therefore we can guarantee that we have $\sum_{a\in[K]}\pi_{a,t}\left(\sum_{j\in[N]}(1-\hat{\mu}_{j,a,t})\right) \le 1 + 2 \ln T$:
%Since this is satisfied for $\pi^*$ on our conditioning in lemma \ref{lemma:highreward_star}
\begin{lemma}
\[\sum_{a\in[K]}\pi_{a,t}\left(\sum_{j\in[N]}(1-\hat{\mu}_{j,a,t})\right) \le 1 + 2 \ln T \] for all $t \ge 180N^2K\ln(6NKT/\delta)\ln T$ conditioned on the events in lemma \ref{lemma:highreward_star} and \ref{lem:error_sum_bound}.
\label{lemma:highreward_muhat_app}
\end{lemma}

Additionally, we repeat the analysis from \ref{lemma:highreward_muhat_app1} to obtain
\begin{lemma}
    \[\sum_{a\in[K]}\pi_{a,t}\left(\sum_{j\in[N]}(1-\mu^*_{j,a})\right) \le 2 + 2 \ln T .\]
    \label{lemma:highreward_pit_mustar}
\end{lemma}

At this point we have all of the tools necessary to continue.

\begin{lemma}
\label{lem:high_order_bound}

If $N_{a,t}\ge180N^2 (1 + \log N) \ln(6NTK/\delta)\ln T$ for all $a$, then

\[\left|\sum_{m=2}^{N}\sum_{B\in\{S:S\subset[N]\wedge|S|=m\}}\left(\prod_{j\in B}b_{j,t}(\pi)\prod_{j'\notin B}a_{j'}(\pi)\right)\right| \le \frac{1}{20\sqrt{N}/19-1}\left(\sum_{j\in[N]}\sum_{a\in[K]}\pi_{a}w_{j,a,t}\right).\]

conditioned on the event in adapted lemma \ref{lemma:LPfair_UCB_param_in_CR_main}.
\end{lemma}
\begin{proof}
For higher-order terms ($2\le|S|\le N$), we have
\begin{align*}
\left|\sum_{m=2}^{N}\sum_{B\in\{S:S\subset[N]\wedge|S|=m\}}\prod_{j\in B}b_{j}(\pi)\prod_{j'\notin B}a_{j'}(\pi)\right| & \le\sum_{m=2}^{N}\sum_{B\in\{S:S\subset[N]\wedge|S|=m\}}\prod_{j\in B}|b_{j}(\pi)|\\
 & \le\sum_{m=2}^{N}\left(\sum_{j\in[N]}|b_{j}(\pi)|\right)^{m}\\
 & =\left(\sum_{j\in[N]}|b_{j}(\pi)|\right)\sum_{m=2}^{N}\left(\sum_{j\in[N]}|b_{j}(\pi)|\right)^{m-1}
\end{align*}
We start by bounding
$\sum|b_{j}|$:
\begin{align*}
\sum_{j\in[N]}|b_{j}(\pi)| & =\sum_{j\in[N]}\left|\sum_{a\in[K]}\pi_{a}(\hat{\mu}_{j,a,t}-\mu_{j,a}^{*})\right|\\
 & \le\sum_{j\in[N]}\sum_{a\in[K]}\pi_{a}w_{j,a,t}\\
 & =\sum_{j\in[N]}\sum_{a\in[K]}\pi_a\left(\sqrt{\frac{12(1-\hat{\mu}_{j,a,t})\ln(6NKT/\delta)}{N_{a,t}}}+\frac{12\ln(6NKT/\delta)}{N_{a,t}}\right)\\
\end{align*}
We analyze the two summands on the innermost term separately. First, the right summand yields
\begin{align*}
    \sum_{j\in[N]}\sum_{a\in[K]}\pi_a\left(\frac{12\ln(6NKT/\delta)}{N_{a,t}}\right) &= N \sum_{a\in[K]}\pi_a\left(\frac{12\ln(6NKT/\delta)}{N_{a,t}}\right) \\
    & \le N \left(\frac{1}{2N^2 \ln T}\right)\\
    &\le 1/(2N)
\end{align*}
which also gives us the bound $1/(2\sqrt{N})$. The left summand gives
\begin{align*}
    \sum_{j\in[N]}\sum_{a\in[K]}\pi_a\left(\sqrt{\frac{12(1-\hat{\mu}_{j,a,t})\ln(6NKT/\delta)}{N_{a,t}}}\right) =& \sum_{a\in[K]}\sum_{j\in[N]}\pi_a \cdot \sqrt{12\ln(6NKT/\delta)}\cdot\sqrt{\frac{(1-\hat{\mu}_{j,a,t})}{N_{a,t}}}\\
    \le& \sum_{a\in[K]}\pi_a\sum_{j\in[N]}\left(\frac{1}{12\sqrt{N}\ln T}(1-\hat{\mu}_{j,a,t})\right)\\ &+ \sum_{a\in[K]}\pi_a\sum_{j\in[N]}\left( 36\cdot\ln(6NKT/\delta)\sqrt{N}\ln T\frac{1}{N_{a,t}}\right)\\
    \le& \frac{\sum_{a\in[K]}\sum_{j\in[N]}\pi_a(1-\hat{\mu}_{j,a,t})}{12\sqrt{N}\ln T} + \frac{36\sqrt{N}\ln T \ln(6NKT/\delta) * N}{180\ln(6NKT/\delta)N^2\ln T}\\
    \le& \frac{1 + 2 \ln T}{12\sqrt{N}\ln T} + \frac{1}{5\sqrt{N}}\\
    =& \frac{1}{12\sqrt{N}\ln T} + \frac{1}{6\sqrt{N}} + \frac{1}{5\sqrt{N}}
\end{align*}
where the first inequality is an application of Young's inequality with $z = 6\sqrt{12N\ln(6NKT/\delta)}\ln T$ and the third inequality holds for $\pi = \pi^*$ and $\pi = \pi_t$ because of lemmas \ref{lemma:highreward_muhat_app1} and \ref{lemma:highreward_muhat_app} respectively. This is bounded by $\frac{9}{20\sqrt{N}}$, which combined with the other bound, $1/2N$, we have the bound $19/(20\sqrt{N})$. Since $N \ge 1$, this bound is at most 19/20 < 1. We use the earlier
bound $\sum_{j}|b_{j}(\pi)| \le \sum_{j}\sum_{a\in[K]}\pi_{a}w_{j,a,t}$ and include the sum over the order of the terms $m$:
\begin{align*}
\left|\sum_{m=2}^{N}\sum_{B\in\{S:S\subset[N]\wedge|S|=m\}}\left(\sum_{j\in B}b_{j,t}(\pi)\prod_{j'\notin B}a_{j'}(\pi)\right)\right| & \le\sum_{m=2}^{N}\left(\sum_{j\in[N]}|b_{j,t}(\pi)|\right)^{m}\\
 & \le\sum_{m=2}^{N}\left(\left(19N^{-1/2}/20\right)^{m-1}\sum_{j\in[N]}\sum_{a\in[K]}|b_{j,t}(\pi)|\right)\\
 & \le\sum_{m=2}^{N}\left(\left(19N^{-1/2}/20\right)^{m-1}\sum_{j\in[N]}\sum_{a\in[K]}\pi_{a}w_{j,a,t}\right)\\
 & \le\frac{(19N^{-1/2}/20)}{1-(19N^{-1/2}/20)}\sum_{j\in[N]}\sum_{a\in[K]}\pi_{a}w_{j,a,t}\\
 & =\frac{1}{20\sqrt{N}/19-1}\sum_{j\in[N]}\sum_{a\in[K]}\pi_{a}w_{j,a,t}.
\end{align*}
\end{proof}

This gives us the bound on the higher order terms, so we continue with the lower order terms.

\begin{lemma}
\label{lem:sum_bj_bound}
\[\left|\sum_{j\in[N]}b_{j,t}(\pi)\right| \le 4\sqrt{\ln(6KT/\delta)}\left(\sqrt{\frac{K}{T}}\sum_{a\in[K]}\pi_{a}\sum_{j\in[N]}(1-\hat{\mu}_{j,a,t})+\sqrt{\frac{T}{K}}\sum_{a\in[K]}\pi_{a}/N_{a,t}\right)\] conditioned on the events of lemma \ref{lem:error_sum_bound}.
\end{lemma}
\begin{proof}
\begin{align*}
\left|\sum_{j\in[N]}b_{j,t}(\pi)\right| & =\left|\sum_{j\in[N]}\left(\sum_{a\in[K]}\pi_{a}(\hat{\mu}_{j,a,t}-\mu_{j,a}^{*})\right)\right|\\
 & =\left|\sum_{a\in[K]}\pi_{a}\left(\sum_{j\in[N]}(\hat{\mu}_{j,a,t}-\mu_{j,a}^{*})\right)\right|\\
 & \le\sum_{a\in[K]}\pi_{a}\left|\sum_{j\in[N]}(\hat{\mu}_{j,a,t}-\mu_{j,a}^{*})\right|\\
 & \le\sum_{a\in[K]}\pi_{a}\sqrt{\frac{4\ln(6KT/\delta)\cdot\sum_{j\in[N]}(1-\mu_{j,a}^{*})}{N_{a,t}}}\\
 & \le2\sqrt{\ln(6KT/\delta)}\left(\text{\ensuremath{\frac{1}{z}}}\sum_{a\in[K]}\pi_{a}\sum_{j\in[N]}(1-\mu_{j,a}^{*})+z\sum_{a\in[K]}\pi_{a}/N_{a,t}\right)
\end{align*}
where the second-to-last inequality follows from lemma \ref{lem:error_sum_bound} and last inequality is an application of Young's Inequality which holds for $z > 0$. It is important that we remove the algorithmically unknown term $\mu^*_{j,a}$ from our bound, so we split that term.
\begin{align*}
\left|\sum_{j\in[N]}b_{j,t}(\pi)\right| & \le2\sqrt{\ln(6KT/\delta)}\left(\text{\ensuremath{\frac{1}{z}}}\sum_{a\in[K]}\pi_{a}(\sum_{j\in[N]}(1-\hat{\mu}_{j,a,t})+\sum_{j\in[N]}(\hat{\mu}_{j,a,t}-\mu_{j,a}^{*}))+z\sum_{a\in[K]}\pi_{a}/N_{a,t}\right)\\
 & \le2\sqrt{\ln(6KT/\delta)}\left(\text{\ensuremath{\frac{1}{z}}}\sum_{a\in[K]}\pi_{a}\left(\sum_{j\in[N]}(1-\hat{\mu}_{j,a,t})+\left|\sum_{j\in[N]}(\hat{\mu}_{j,a,t}-\mu_{j,a}^{*})\right|\right)+z\sum_{a\in[K]}\pi_{a}/N_{a,t}\right)
\end{align*}
Note that the second term is an earlier step in the bound with a factor $\frac{2\sqrt{\ln(6KT/\delta)}}{z}$, specifically from the first inequality in the proof. Therefore,
\begin{align*}
\left|\sum_{j\in[N]}b_{j,t}(\pi)\right|\le\sum_{a\in[K]}\pi_{a}\left|\sum_{j\in[N]}(\hat{\mu}_{j,a,t}-\mu_{j,a}^{*})\right| & \le2\sqrt{\ln(6KT/\delta)}\Big(\text{\ensuremath{\frac{1}{z}}}\sum_{a\in[K]}\pi_{a}\sum_{j\in[N]}(1-\hat{\mu}_{j,a,t}) \\ &+\frac{1}{z}\sum_{a\in[K]}\pi_{a,t}\Big|\sum_{j\in[N]}(\hat{\mu}_{j,a,t}-\mu_{j,a}^{*})\Big|+z\sum_{a\in[K]}\pi_{a}/N_{a,t}\Big)\\
\implies \left(1-\frac{2\sqrt{\ln(6KT/\delta)}}{z}\right)\sum_{a\in[K]}\pi_{a}\left|\sum_{j\in[N]}(\hat{\mu}_{j,a,t}-\mu_{j,a}^{*})\right| & \le2\sqrt{\ln(6KT/\delta)}\left(\frac{1}{z}\sum_{a\in[K]}\pi_{a}\sum_{j\in[N]}(1-\hat{\mu}_{j,a,t})+z\sum_{a\in[K]}\pi_{a}/N_{a,t}\right)\\
\implies \left|\sum_{j\in[N]}b_{j,t}(\pi)\right|\le\sum_{a\in[K]}\pi_{a}\left|\sum_{j\in[N]}(\hat{\mu}_{j,a,t}-\mu_{j,a}^{*})\right| & \le\frac{2z\sqrt{\ln(6KT/\delta)}}{z-2\sqrt{\ln(6KT/\delta)}}\left(\frac{1}{z}\sum_{a\in[K]}\pi_{a}\sum_{j\in[N]}(1-\hat{\mu}_{j,a,t})+z\sum_{a\in[K]}\pi_{a}/N_{a,t}\right)\\
 & \le4\sqrt{\ln(6KT/\delta)}\left(\sqrt{\frac{K}{T}}\sum_{a\in[K]}\pi_{a}\sum_{j\in[N]}(1-\hat{\mu}_{j,a,t})+\sqrt{\frac{T}{K}}\sum_{a\in[K]}\pi_{a}/N_{a,t}\right)
\end{align*}
where we use $z=\sqrt{T/K}$ and $T\ge16K\sqrt{\ln(6KT/\delta)}$.
\end{proof}

\begin{lemma}
\label{lem:sqrt_sumsquared_bound}
If $N_{a,t}\ge180N^2 \ln(6NTK/\delta)\ln T$ for all $a$, then
\[\sqrt{\sum_{j\in[N]}b^2_{j,t}(\pi)} \le \left(\frac{\sqrt{T}}{\sqrt{2}\sqrt{K}}+12\sqrt{2}\sqrt{N}\ln(6NKT/\delta)\right)\sum_{a\in[K]} \frac{\pi_a}{N_{a,t}}+\frac{6\sqrt{2}\sqrt{K}\ln(6NKT/\delta)}{\sqrt{T}}\sum_{a\in[K]}\sum_{j\in[N]}\pi_a(1-\hat{\mu}_{j,a,t})\]
conditioned on the event in the adapted lemma \ref{lemma:LPfair_UCB_param_in_CR_main}.
\end{lemma}
\begin{proof}

\begin{align*}
 \sqrt{\sum_{j\in[N]}b_{j}^{2}(\pi)} & =\sqrt{\sum_{j\in[N]}\big(\sum_{a\in[K]}\pi_{a}(\hat{\mu}_{j,a,t}-\mu_{j,a}^{*})\big)^{2}}\\
 & =\sqrt{\sum_{j\in[N]}\sum_{a_{1},a_{2}\in[K]}\pi_{a_{1}}\pi_{a_{2}}(\hat{\mu}_{j,a_{1},t}-\mu_{j,a_{1}}^{*})(\hat{\mu}_{j,a_{2},t}-\mu_{j,a_{2}}^{*})}\\
 & =\sqrt{\sum_{a_{1},a_{2}\in[K]}\pi_{a_{1}}\pi_{a_{2}}\sum_{j\in[N]}(\hat{\mu}_{j,a_{1},t}-\mu_{j,a_{1}}^{*})(\hat{\mu}_{j,a_{2},t}-\mu_{j,a_{2}}^{*})}\\
 & \le\sqrt{\sum_{a_{1},a_{2}\in[K]}\pi_{a_{1}}\pi_{a_{2}}\sqrt{\left(\sum_{j\in[N]}(\hat{\mu}_{j,a_{1},t}-\mu_{j,a_{1}}^{*})^{2}\right)\left(\sum_{j\in[N]}(\hat{\mu}_{j,a_{2},t}-\mu_{j,a_{2}}^{*})^{2}\right)}}\\
 & =\sqrt{\left(\sum_{a\in[K]}\pi_{a}\sqrt{\sum_{j\in[N]}(\hat{\mu}_{j,a.t}-\mu_{j,a}^{*})^{2}}\right)^{2}}\\
 & =\sum_{a\in[K]}\pi_{a}\sqrt{\sum_{j\in[N]}(\hat{\mu}_{j,a.t}-\mu_{j,a}^{*})^{2}}\\
 & \le\sum_{a\in[K]}\pi_{a}\sqrt{\sum_{j\in[N]}w_{j,a,t}^{2}}\\
 & =\sum_{a\in[K]}\pi_a\sqrt{\sum_{j\in[N]}\left(\sqrt{\frac{12 (1 - \hat{\pa}_{j, a, t}) \ln(6NKT/\delta)}{\st_{a,t}}}  +  \frac{12  \ln(6NKT/\delta)}{\st_{a,t}}\right)^2}\\
 & \le\sum_{a\in[K]}\pi_a\sqrt{2\sum_{j\in[N]}\left(\frac{12 (1 - \hat{\pa}_{j, a, t}) \ln(6NKT/\delta)}{\st_{a,t}}  +  \left(\frac{12  \ln(6NKT/\delta)}{\st_{a,t}}\right)^2\right)}\\
 & \le\sqrt{2}\sum_{a\in[K]}\pi_a\left(\sqrt{\sum_{j\in[N]}\frac{12 (1 - \hat{\pa}_{j, a, t}) \ln(6NKT/\delta)}{\st_{a,t}}}  + \sqrt{\sum_{j\in[N]} \left(\frac{12  \ln(6NKT/\delta)}{\st_{a,t}}\right)^2}\right)\\
 %& \le \frac{1}{\sqrt{N}}\sum_{a\in[K]}\sum_{j\in[N]}\pi_aw_{j,a,t}
\end{align*}
The first inequality holds by the Cauchy-Schwarz inequality, the second holds by the adapted lemma \ref{lemma:LPfair_UCB_param_in_CR_main}, the third holds because $2(x^2 + y^2) \ge (x+y)^2$, and the last holds because $\sqrt{x + y} \le \sqrt{x} + \sqrt{y}$ for all $x,y\ge 0$.  We bound these sums separately:
\begin{align*}
    \sum_{a\in[K]}\pi_a\sqrt{\sum_{j\in[N]}\frac{12 (1 - \hat{\pa}_{j, a, t}) \ln(6NKT/\delta)}{\st_{a,t}}} &\le \frac{z}{2}\sum_{a\in[K]}\frac{\pi_a}{N_{a,t}} + \frac{1}{2z}\sum_{a\in[K]}\pi_a\sum_{j\in[N]}12 (1 - \hat{\pa}_{j, a, t}) \ln(6NKT/\delta)\\
    &=\frac{\sqrt{T}}{2\sqrt{K}}\sum_{a\in[K]} \frac{\pi_a}{N_{a,t}}+\frac{6\sqrt{K}\ln(6NKT/\delta)}{\sqrt{T}}\sum_{a\in[K]}\sum_{j\in[N]}\pi_a(1-\hat{\mu}_{j,a,t})\\
    \sum_{a\in[K]}\pi_a\sqrt{\sum_{j\in[N]}\left(\frac{12\ln(6NKT/\delta}{N_{a,t}}\right)^2} &= 12\sqrt{N}\ln(6NKT/\delta)\sum_{a\in[K]}\frac{\pi_a}{N_{a,t}}
\end{align*}
where we use Young's inequality in the first sum.
Combining these bounds gives the lemma.

\end{proof}

\begin{lemma}
\label{lem:low_order_bound}
\begin{align*}
    \left(\prod_{j=1}^{N}a_{j}(\pi)\right)\left(\sum_{j\in[N]}\frac{b_{j}(\pi)}{a_{j}(\pi)}\right) \le& \left(4\sqrt{\ln(6KT/\delta)} + 6\sqrt{2}\ln(6NKT/\delta)\sqrt{2+2\ln T}\right)\sqrt{\frac{K}{T}}\sum_{a\in[K]}\pi_{a}\sum_{j\in[N]}(1-\hat{\mu}_{j,a,t}) \\
    &+\Big(\left(4\sqrt{\ln(6KT/\delta)} + \sqrt{1+\ln T} \right)\sqrt{\frac{T}{K}}\\& + 12\sqrt{2}\sqrt{N}\ln(6NKT/\delta)\sqrt{2+2\ln T}\Big)\sum_{a\in[K]}\pi_{a}/N_{a,t}% + \frac{1}{T^2}
\end{align*}
\end{lemma}

\begin{proof}
We begin our bound on first-order ($|S|=1$) terms by separating into two sums. %Note that by applying lemma~\ref{lem:low_utility} to $a_1(\pi),\ldots,a_N(\pi)$,  we have either $\left(\prod_{j=1}^{N}a_{j}(\pi)\right) \leq \frac{1}{T^2}$ or $\sum_{j\in[N]}1-a_{j}(\pi) \in O(\log N \log T)$. In the first case, we have. 

%\begin{align*}
    %\left(\prod_{j=1}^{N}a_{j}(\pi)\right)\left(\sum_{j\in[N]}\frac{b_{j}(\pi)}{a_{j}(\pi)}\right) & 
  %\le   \frac{1}{T^2}\\
%\end{align*}
\begin{align*}
\left(\prod_{j=1}^{N}a_{j}(\pi)\right)\left(\sum_{j\in[N]}\frac{b_{j}(\pi)}{a_{j}(\pi)}\right) & 
  \le \left(\prod_{j=1}^{N}a_{j}(\pi)\right) \left(\left|\sum_{j\in[N]}b_{j}(\pi)\right|+\sum_{j\in[N]}|b_{j}(\pi)||1/a_{j}(\pi)-1| \right) \\%+  \frac{1}{T^2}\\
 & \le \left(\prod_{j=1}^{N}a_{j}(\pi)\right)\left| \sum_{j\in[N]}b_{j}(\pi)\right|+ \sum_{j\in[N]}|b_{j}(\pi)||1-a_{j}(\pi)|\\% +  \frac{1}{T^2} \\
 & \le\left|\sum_{j\in[N]}b_{j}(\pi)\right|+\sqrt{\left(\sum_{j\in[N]}b_{j}^{2}(\pi)\right)\left(\sum_{j\in[N]}(1-a_{j}(\pi))^{2}\right)}\\% +  \frac{1}{T^2} \\
 & \le\left|\sum_{j\in[N]}b_{j}(\pi)\right|+\sqrt{\left(\sum_{j\in[N]}b_{j}^{2}(\pi)\right)\sum_{j\in[N]}1-a_{j}(\pi)}\\% +  \frac{1}{T^2} \\
 & \le\left|\sum_{j\in[N]}b_{j}(\pi)\right|+\sqrt{2 + 2 \ln T } \cdot \sqrt{\sum_{j\in[N]}b_{j}^{2}(\pi)}\\% +  \frac{1}{T^2}
\end{align*}
where the second inequality holds because $|1/a_j(\pi) - 1| = |1-a_j(\pi)| / a_j(\pi)$ and the third inequality is an application of Cauchy-Schwarz. We then combine with the bounds from lemmas \ref{lem:sum_bj_bound} and \ref{lem:sqrt_sumsquared_bound} to obtain our bound. Note that our $\sqrt{2 + 2\ln T}$ term comes from \ref{lemma:highreward_star} for $\pi^*$ and \ref{lemma:highreward_pit_mustar} for $\pi_t$.
\end{proof}

We are now ready to prove Theorem \ref{theorem:alg3_regret_app}:

\begin{proof}
\begin{align*}
\left|\text{NSW}(\ensuremath{\pi,}\ensuremath{\hat{\mu}_{t}})-\text{NSW}(\ensuremath{\pi},\ensuremath{\mu^{*}})\right|  =&\left|\prod_{j=1}^{N}(a_{j}(\pi)+b_{j,t}(\pi))-\prod_{j=1}^{N}a_{j}(\pi)\right|\\
  \le&\left|\sum_{j\in[N]}b_{j,t}(\pi)\right|+\sqrt{2+2\ln T }\sqrt{\sum_{j\in[N]}b_{j,t}^{2}(\pi)}+\sum_{m=2}^{N}\left(\sum_{j\in[N]}|b_{j,t}(\pi)|\right)^{m}\\
  \le& \left(4\sqrt{\ln(6KT/\delta)} + 6\sqrt{2}\ln(6NKT/\delta)\sqrt{2+2\ln T }\right)\sqrt{\frac{K}{T}}\sum_{a\in[K]}\pi_{a}\sum_{j\in[N]}(1-\hat{\mu}_{j,a,t}) \\
    &+\Big(\Big(4\sqrt{\ln(6KT/\delta)} + \sqrt{1+\ln T } \Big)\sqrt{\frac{T}{K}} \\&+ 12\sqrt{2}\sqrt{N}\ln(6NKT/\delta)\sqrt{2+2\ln T }\Big)\sum_{a\in[K]}\pi_{a}/N_{a,t}\\
  &+\frac{1}{20\sqrt{N}/19-1}\sum_{j\in[N]}\sum_{a\in[K]}\pi_{a}w_{j,a,t}\\
  =&f_{t}(\pi,\hat{\mu}_{t})
\end{align*}
gives us a bound on the error of the Nash social welfare for a policy between the true and estimated rewards. Note that all of the values in this bound, $N$, $K$, $T$, $N_{a,t}$, $\hat{\mu}$, and $w$, are available to the algorithm at iteration $t$. Additionally, we see that all of these terms can be rewritten as the dot product of a vector with $\pi$. Therefore, we are able to optimize $f_t(\pi,\hat{\mu}_t)$ by computing the vector $\eta_t$ as in algorithm \ref{alg:fair_UCB_inefficient_app}.
 We define $\pi_{t}$ at each time step $t>180N^2K\ln(6NTK/\delta)\ln T$ as
\[
\pi_{t}=\arg\max_{\pi\in\Delta_{K}}\text{NSW}(\pi,\hat{\mu}_{t})+f_{t}(\pi,\hat{\mu}_{t})
\]
Then we have at each time step $t>180N^2K\ln(6NTK/\delta)\ln T$:
\begin{align*}
\text{NSW(\ensuremath{\pi^{*},\mu^{*})}} & \le\text{NSW}(\ensuremath{\pi^{*}},\hat{\mu}_{t})+f_{t}(\pi^{*},\hat{\mu}_{t})\\
 & \le\text{NSW}(\pi_{t},\hat{\mu}_{t})+f_{t}(\pi_{t},\hat{\mu}_{t})\\
 & \le\text{NSW}(\pi_{t},\mu^{*})+2f_{t}(\pi_{t},\hat{\mu}_{t}),
\end{align*}
where the second bound holds because we condition on $\pi^* \in S_\pi$, which shows that our instantaneous regret at time $t$ is bounded
by $\text{NSW}(\ensuremath{\pi^{*},\mu^{*})}-\text{NSW}(\pi_{t},\mu^{*})\le2f_{t}(\pi_{t},\hat{\mu}_{t})$. At this point we recall that we condition on the event in lemma \ref{lemma:highreward_star}, otherwise we have the trivial regret bound $T^{-1}$ by lemma \ref{lemma:trivial_regret_app}.

We first pull each arm $\tilde{O}(N^{2})$ times,
each with an instantaneous regret of at most 1, for a $\tilde{O}(N^{2}K)$
term in the regret bound. The remainder of the regret bound comes
from bounding the sum $2\sum_{t\in[T_{0},T]}f_{t}(\pi_{t},\hat{\mu}_{t})$ over all other rounds:

\begin{align*}
\sum_{t}f_{t}(\pi_{t},\hat{\mu}_{t}) =& \left(4\sqrt{\ln(6KT/\delta)} + 6\sqrt{2}\ln(6NKT/\delta)\sqrt{2+2\ln T }\right)\sqrt{\frac{K}{T}}\sum_{t}\sum_{a\in[K]}\pi_{a}\sum_{j\in[N]}(1-\hat{\mu}_{j,a,t}) \\
    &+\Big(\left(4\sqrt{\ln(6KT/\delta)} + \sqrt{1+\ln T } \right)\sqrt{\frac{T}{K}} \\&+ 12\sqrt{2}\sqrt{N}\ln(6NKT/\delta)\sqrt{2+2\ln T }\Big)\sum_{a\in[K]}\sum_{t}\pi_{a}/N_{a,t}\\
  &+\frac{1}{20\sqrt{N}/19-1}\sum_{t}\sum_{j\in[N]}\sum_{a\in[K]}\pi_{a}w_{j,a,t}\\
  \le&\tilde{O}(\sqrt{KT}) + \tilde{O}(\sqrt{KT} + K\sqrt{N}) + \tilde{O}\left(\frac{\sqrt{NKT}}{\sqrt{N}}\right)
\end{align*}
where we bound $\sum_{a\in[K]}\sum_{j\in[N]}\pi_{a,t}(1-\hat{\mu}_{j,a,t})\le1+2\ln T$
using lemma \ref{lemma:highreward_muhat_app}, we bound $\sum_{t}\sum_{a\in[K]}\pi_{a,t}/N_{a,t}$
using lemma~\ref{lemma:LPfair_UCB_concentration_width_nonsqr_main}, and we bound $\sum\sum\sum\pi_{a,t}w_{j,a,t}$ using
the analysis in theorem~\ref{theorem:main_main} (equations ~\ref{hard_main},~\ref{hard1_main},~\ref{hard2_main}) with lemma \ref{lemma:highreward_muhat_app} in the place of \ref{lemma:highreward_main}.
Putting this together with the trivial regret bound for the first $T_0 = 180N^2K\ln(6NTK/\delta)\ln T$ rounds gives us the entire bound $\tilde{O}(\sqrt{KT} + N^2K)$. We obtain our failure probability from the events we condition on in lemma \ref{lem:error_sum_bound} and the adapted lemmas \ref{lemma:LPfair_UCB_param_in_CR_main} and \ref{lemma:LPfair_UCB_concentration_width_nonsqr_main}, which each have failure probability $\delta/3$. Hence, with probability at least $(1-\delta)$
we have regret at most $\tilde{O}(\sqrt{KT}+N^{2}K)$. 

\end{proof}

\section{Additional Experimental Data}

Our error bounds tend to be quite large, since changing the value of $\mu^*$ significantly affects the resulting regret. We include them here for completeness, reported as mean plus/minus standard deviation.

\begin{table}[ht]
\caption{Average cumulative regrets over 10 values of $\mu^*$, with $T = 5 \cdot 10^5$}
\centering
\begin{tabular}{cc|cc|cc|c}
$\np$ & $\na$ & \multicolumn{2}{c|}{Algorithm \ref{alg:fair_UCB}} & \multicolumn{2}{c|}{\cite{hossain2020fair}} & $\F(\po^\star, \mu^\star )$ \\ \hline
   &   & $t = 2 \cdot 10^5 $ & $t = 5\cdot 10^5 $ & $t = 2 \cdot 10^5$ & $t = 5\cdot 10^5 $ &                     \\
4  & 2 & $ 1222 \pm 127 $      & $ 1806 \pm 501 $      & $ 1226 \pm 201 $      & $ 1832 \pm 549 $      & $0.9428 \pm 0.0370$ \\
20 & 4 & $ 8313 \pm 2211 $      & $ 15322 \pm 4000 $     & $ 10801 \pm 2975 $     & $21655 \pm 5690 $      & $0.6353 \pm 0.0358$ \\
80 & 8 & $ 4521 \pm 1500 $      & $ 9966 \pm 3049 $      & $ 5230 \pm 1800 $      & $ 12874 \pm 4398 $     & $0.0808 \pm 0.0154$
\end{tabular}
\label{tab:appendix_nsw}
\end{table}

\begin{table}[ht!]
\caption{Average running times over 10 values of $\mu^*$, with $T = 5 \cdot 10^5$}
\centering
\begin{tabular}{cc|c|c}
$\np$ & $\na$ & Algorithm \ref{alg:fair_UCB} & \cite{hossain2020fair} \\ \hline
4 & 2 & $365 \pm 12$ & $802 \pm 179$ \\ \hline
20 & 4 & $675 \pm 49$ & $2324 \pm 496$ \\ \hline
80 & 6 & $1862 \pm 65$ & $14128 \pm 629$ \\ \hline
\end{tabular}
\label{tab:appendix_times}
\end{table}

We also include time measurements for the algorithms here, reported in seconds. Recall that the running time is tuneable by controlling the stopping conditions of the gradient ascent, these are the times reported for the same settings as the results in the Experiments section.

The experiments were run without the GPU, on a system with an Intel i5-11600K CPU and a Radeon RX 580 GPU. We omit experiments for algorithm \ref{alg:fair_UCB_inefficient} because it performs poorly in practice due to its large constants and prohibitive start-up cost in terms of $N$, although it is of theoretical interest.